%% file: main.tex
\documentclass{article}

 \usepackage[main, preprint, nonatbib]{neurips_2026}

\newcommand{\citet}[1]{\cite{#1}}
\newcommand{\citep}[1]{\cite{#1}}

\usepackage[utf8]{inputenc} %
\usepackage[T1]{fontenc}    %
\usepackage[
    colorlinks=true,
    linkcolor=blue!55!black,
    citecolor=blue!45!black,
    urlcolor=purple!60!black,
    filecolor=magenta!60!black,
    bookmarks=true,
    bookmarksopen=true,
    bookmarksnumbered=true,
    unicode=true
]{hyperref}
\usepackage{url}            %
\usepackage{booktabs}       %
\usepackage{amsfonts}       %
\usepackage{nicefrac}       %
\usepackage{microtype}      %

\RequirePackage{fancyhdr}
\RequirePackage[dvipsnames]{xcolor} %

\usepackage{algorithm}
\usepackage{algorithmic}

\usepackage{amsmath,amsthm,amssymb,mathtools}
\usepackage{thmtools} %
\usepackage{nicefrac}
\usepackage{wrapfig}
\usepackage{multirow}
\usepackage{nccmath}
\usepackage[capitalize]{cleveref}

\input{preambule}

\title{From Inexact Gradients to Byzantine Robustness: \\ Acceleration and Optimization under Similarity}

\author{
Renaud Gaucher$^{1,2,*}$ \and \textbf{Aymeric Dieuleveut}$^{1}$ \and \textbf{Hadrien Hendrikx}$^{2}$\\
$^1$ CMAP, École polytechnique, Institut Polytechnique de Paris, Palaiseau, France\\
$^2$ Inria, Center of Grenoble Alpes University, CNRS, LJK, Grenoble, France\\
$^*$ Corresponding author: renaudp[dot]gaucher[at]gmail[dot]com
}

\begin{document}

\maketitle

\begin{abstract}
Distributed learning algorithms are vulnerable to adversarial nodes, a.k.a. Byzantine failures. To solve this issue, robust algorithms have been developed, which typically replace parameter averaging with robust aggregations. While generic conditions on these aggregations exist to guarantee the convergence of (Stochastic) Gradient Descent (SGD), the analyses remain rather ad hoc. This hinders the development of more complex robust algorithms, such as accelerated ones.
In this work, we show that Byzantine-robust distributed optimization can, under standard generic assumptions, be cast as a general optimization with inexact gradient oracles, an active field of research. This allows us to obtain state-of-the-art results for Byzantine-robust optimization from general inexact first-order analyses. 

We first show that inexact GD on top of standard robust aggregation procedures obtains optimal asymptotic error in the Byzantine setting.
Going further, we study an algorithm for \emph{Optimization under Similarity}, in which the server leverages an auxiliary loss function that approximates the global loss. Then, we introduce an \emph{Accelerated Extra-gradient} method that yields acceleration in both the standard and similarity settings. We first give new general convergence results for these inexact schemes, improving the tolerance level of acceleration to a relative error term by a $\kappa^{\nicefrac 12}$ factor over previous analysis, and then instantiate these results in the Byzantine setting through our reduction.
Both algorithms drastically reduce the communication complexity, as we show theoretically and empirically. 
\end{abstract}

\input{core.tex}

\bibliographystyle{alpha}
\bibliography{biblio}

\appendix

\tableofcontents

\input{appendix}

\end{document}

%% file: preambule.tex
\usepackage{nicefrac}

\usepackage[textsize=tiny, backgroundcolor=white]{todonotes}

\newcommand{\renaud}[1]{}
\newcommand{\aymeric}[1]{}
\newcommand{\hadrien}[1]{}

\newcommand{\new}[1]{#1}

\newcommand{\vlong}[1]{}
\newcommand{\vshort}[1]{#1}

\input{math_command}

\newcommand{\clip}{\mathrm{clip}}

\newcommand{\norm}[1]{\|#1\|}
\newcommand{\sqn}[1]{\norm{#1}^2}

\newcommand{\paren}[1]{{\left( {#1} \right)}}

\newcommand{\cB}{\mathcal{B}}

\newcommand{\cD}{\mathcal{D}}

\newcommand{\cG}{\mathcal{G}}
\newcommand{\cH}{\mathcal{H}}

\newcommand{\cL}{\mathcal{L}}

\newcommand{\cN}{\mathcal{N}}
\newcommand{\cO}{\mathcal{O}}

\theoremstyle{plain}
\newtheorem{theorem}{Theorem}[section]
\newtheorem{proposition}[theorem]{Proposition}
\newtheorem{lemma}[theorem]{Lemma}
\newtheorem{corollary}[theorem]{Corollary}
\theoremstyle{definition}
\newtheorem{definition}[theorem]{Definition}
\newtheorem{assumption}[theorem]{Assumption}
\theoremstyle{remark}
\newtheorem{remark}[theorem]{Remark}

%% file: math_command.tex
\usepackage{amsmath,amsfonts,bm}

\usepackage{bbm}

\def\eqref#1{equation~\ref{#1}}

\def\1{\bm{1}}

\def\eps{{\varepsilon}}

\def\ve{{\bm{e}}}

\def\vg{{\bm{g}}}

\def\vu{{\bm{u}}}
\def\vv{{\bm{v}}}

\def\vx{{\bm{x}}}
\def\vy{{\bm{y}}}
\def\vz{{\bm{z}}}

\DeclareMathAlphabet{\mathsfit}{\encodingdefault}{\sfdefault}{m}{sl}
\SetMathAlphabet{\mathsfit}{bold}{\encodingdefault}{\sfdefault}{bx}{n}

\newcommand{\R}{\mathbb{R}}

\DeclareMathOperator*{\argmin}{arg\,min}

%% file: core.tex
\section{Introduction}

Distributed Learning is a paradigm in which multiple computers collaborate to train a machine learning model. This allows us to simultaneously leverage the computational power of multiple devices and to adapt to the distributed nature of data, which is often stored on smartphones, IOT devices, or by larger stakeholders such as companies or hospitals. However, involving multiple parties in an algorithmic process opens up to new threats: machines can crash, software can be buggy or even hacked, data can be corrupted or even adversarially crafted. Byzantine failure \citep{lamport1982byzantine} is a worst-case model of such behaviors, corresponding to omniscient adversarial parties. 

Distributed Gradient Descent (DGD), the workhorse of modern distributed learning, is vulnerable to such failures \citep{blanchard2017machine}: when the clients' gradients are averaged, one Byzantine node returning adversarially crafted gradients can corrupt the result arbitrarily, thus ruining the optimization. To address this critical issue, a large body of research has been devoted to developing and studying distributed learning methods that converge even with adversaries \citep{blanchard2017machine, yin2018byzantine, chen2017distributed}. Those methods crucially replace the gradients (or local parameters) averaging step with a robust aggregation scheme, such as the geometric median or coordinate-wise trimming. Over the years, analyses specific to each robust aggregation method have given way to generic characterizations and standard study assumptions \citep{allouah2023fixing, karimireddy2020byzantine}. Typically, one simply needs to guarantee that the error (\emph{i.e.}, distance from the robust aggregation result to the average of the honest nodes) remains bounded. This bound depends on both the heterogeneity between honest nodes and the proportion of Byzantine nodes. By controlling these two sources of error, one can determine the optimal asymptotic error, and show that it can be achieved by several robust aggregation schemes \citep{farhadkhani2022byzantine, karimireddy2021learning}. 
Momentum-based methods were then proposed to alleviate the stochastic error in the Byzantine context \citep{karimireddy2021learning, farhadkhani2022byzantine}. Similarly, refined results taking the dimension into account have been achieved by leveraging high-dimensional averaging tools \citep{diakonikolas2019robust, zhu2023byzantine}. Robust distributed learning algorithms were also coupled with gradient compression \citep{rammal2024communication, malinovsky2024byzantine, liu2024byzantine} to reduce their communication cost.

As mentioned above, Byzantine-resilient algorithms work for a wide variety of robust aggregation methods, as long as they satisfy relevant error conditions. Yet, analyses of robust algorithms are still tailored to the Byzantine setting: proofs closely follow existing schemes, but are adapted to match the specificity of the distributed adversarial setting. Although this adaptation is rather transparent in simple settings, it becomes much more cumbersome for more involved proofs, with less leeway to accumulate error terms. This hinders the development of efficient algorithms, leaving critical questions open, and preventing the adaptation of efficient federated algorithms. For instance, \citet{ghosh2020distributed} proposed second-order methods, but they require a very low heterogeneity level to actually converge. Recently, \citet{shi2025optimal} proposed an analysis of Nesterov's accelerated Byzantine-robust method, but only under a restrictive Heterogeneity assumption.
\subsection*{Contributions and outline}

Our work alleviates the aforementioned shortcomings of existing theory. We cast Byzantine-robust optimization as a special case of optimization under inexact (gradient) oracles, and build on this to derive efficient robust algorithms with provably low iteration (and so communication) complexity.  

\textbf{Generic Formalism.} After reviewing the setting in \cref{subsec:setup},  we show in \cref{subsec:reduction} that Byzantine-robust distributed optimization exactly fits into the formalism of first-order methods with inexact gradient oracles. 
To do so, we show that combining $(G,B)$-heterogeneity \citep{karimireddy2020scaffold}, a standard assumptions on the heterogeneity of loss functions, with $(f,\nu)$-robustness \citep{allouah2023fixing}, a standard characterization of the employed robust aggregation method, leads to simple additive and multiplicative error terms on the mean of the honest gradient.
Optimization under the provided inexact gradient formulation is an active field of research, from which we can thus directly reuse many results. \\

We then show in \cref{subsec:csq_inexact_oracles} that this abstraction is tight in the sense that using existing inexact GD results allows us to match the Byzantine-robust optimization lower bound on the asymptotic error under heterogeneity. We also leverage existing work on inexact gradients showing acceleration can be achieved in the Byzantine setting. We specifically identify the condition on the relative inexactness as a major challenge towards understanding the limits of acceleration in the robust setting.

These first results thus both motivate and pave the way for a more systematic search for efficient algorithms, abstracting away the complexity of the Byzantine setting, and thus building on the inexact gradient formulation to propose a robust version of two standard fast algorithms.

\textbf{Optimization under similarity.} When an approximation (in the second-order sense) of the global honest loss is available, we give in \cref{sec:similarity} a Byzantine-robust algorithm (with computationally more expensive updates than GD) called Prox Inexact Gradient method under Similarity (PIGS). PIGS enjoys a linear convergence rate of $\mathcal{O}(\nicefrac{\Delta}{\mu}\log(1/\eps))$ to reach an $\eps$-neighborhood of the optimal error, where $\Delta$ is a measure of the distance between the Hessian of the approximate loss and the true honest one. Since $\Delta$ can be drastically smaller than the smoothness of the loss function, this results in an important gain in iteration complexity, which consequently reduces the communication time. The convergence theory of PIGS requires an extension of~\citet{woodworth2023two} to handle the inexactness of the error term.

\textbf{Acceleration.}  In \cref{sec:extragradient_accelerated}, we adapt the extra-gradient method from \citep{kovalev2022optimal} to the inexact oracle setting. We show that it converges in this setting towards the \textit{optimal asymptotic error} with an \textit{accelerated linear rate} of convergence $\mathcal{O}(\sqrt{\nicefrac{\mu}{L}})$, where $\mu$ is the strong convexity of the considered loss function and $L$ is the smoothness. This result holds as long as the relative error in the oracle is smaller than $\sqrt{\nicefrac{\mu}{L}}$. Our result outperforms previous works on accelerated inexact oracle first-order methods by a $\sqrt{\nicefrac{\mu}{L}}$ factor in the condition on the relative error.\\
The extra-gradient algorithm also enjoys an accelerated rate of convergence in the similarity setting of \cref{sec:similarity}, but only as long as the relative error in the gradient is small with respect to $\Delta^2/L^2$.

\textbf{Experiments.} We empirically show the robustness and performance of the proposed methods by comparing them in the Byzantine setting to existing methods in \cref{sec:experiments}.

As a side contribution, we propose in \cref{app:sec:huber} a new aggregation rule related to the Huber estimator \cite{huber1964robust} and centered clipping \cite{karimireddy2021learning}, which outperforms existing $(f,\nu)$-robust aggregators and nearly matches the lower bound.

\paragraph{Notations.} $\langle \cdot, \cdot \rangle$ denote the standard dot product in $\R^d$,  $\|\cdot\|$ the euclidean norm,  and $\|\cdot\|_{op}$ the associated operator norm in $\R^{d\times d}$. For a real valued function $f:\R^d \rightarrow \R$, $\nabla f$ denotes its gradient, and $\nabla^2 f$ its Hessian. %

\section{Setting and Reduction}

\subsection{Byzantine Robust Distributed Optimization}
\label{subsec:setup}
We consider a server-client distributed system, with $n$ clients. Among them, an unknown subset $\cB \subset [n]$ is composed of $|\cB| = f$ Byzantine clients. We denote $\cH = [n]\backslash \cB$ the $n-f$ honest clients, and $\cL_i : \R^d \rightarrow \R$ the loss function of client~$i$ (\emph{e.g.}, the empirical loss on the dataset stored by client $i$). We aim at finding 
\begin{equation}
\label{eq:byzantine_loss}
    \vx^* =\argmin_{\vx\in\R^d}\left \{ \cL_{\cH}(\vx) := \frac{1}{n-f}\sum_{i \in \cH}\cL_i(\vx)\right\}.
\end{equation}
We assume that all $\cL_i$ are differentiable, convex, and that $\cL_{\cH}$ is $\mu$-strongly convex and $L$-smooth.
\begin{definition}
\label{def:strong_convexity_and_smoothness}
A function $\cL$ is $\mu$\textit{-strongly convex}, and $L$-\textit{smooth} if, \(\forall\, \vx,\vy \in \R^d\),
    $\frac{\mu}{2}\|\vx - \vy\|^2 \le \cL(\vx) - \cL(\vy) - \langle  \nabla \cL(\vy), \vx-\vy\rangle \le \frac{L}{2}\|\vx - \vy\|^2$.
$\cL$ is \textit{convex} if the left inequality holds with $\mu=0$. The ratio $\kappa = L/\mu$ is the \emph{condition number} of $\cL$.
\end{definition}

The goal of (distributed) optimization is to output a model $\hat{\vx}$ with an error that vanishes with the computing budget. We say that an algorithm is resilient if it can reach a given error level $\eps$ despite the presence of adversaries, namely that the output $\hat{\vx}$ satisfies $\cL_{\cH}(\hat{\vx}) - \cL_{\cH}(\vx^*) \le \eps$.

\textbf{Hardness of Byzantine Optimization.}
Byzantine clients cannot be distinguished from honest clients with outlying losses. %
To obtain bounded errors, it is therefore necessary to assume that honest node losses have some level of homogeneity.
\begin{assumption}[(G,B)-bounded heterogeneity]
\label{asmpt:GB_heterogeneity}
    The local loss functions of honest clients are said to satisfy {$(G,B)$-bounded heterogeneity} when, for any $\vx \in \R^d$, 
    \[
    \frac{1}{|\cH|}\sum_{i \in \cH} \|\nabla \cL_i(\vx) - \nabla \cL_{\cH}(\vx)\|^2 \le G^2 + B^2\|\nabla \cL_{\cH}(\vx)\|^2.
    \]
\end{assumption}
The $G^2$ term relates to the heterogeneity at the optimum of the gradients, while, in a worst-case analysis, the $B^2$ can be upper bounded by $2(L/\mu-1)$ when $\cL_{\cH}$ is $\mu$-strongly convex and $(\cL_i)_{i\in\cH}$ are $L$-smooth \cite{allouah2024byzantine}. In Prop. \ref{prop:similarity_to_GB_heter} we alternatively bound $B^2$ using an Hessian similarity metric.

One can then derive the minimum error that any distributed algorithm has to pay, which increases with both the heterogeneity among honest clients and the proportion of Byzantine clients, up to a critical fraction known as the \textit{breakdown point}, at which adversaries can arbitrarily disturb the system.

\begin{theorem}[\citet{allouah2024robust}]
\label{thm:allouha_lower_bound}
  For any distributed algorithm, there exist quadratic local loss functions satisfying \cref{asmpt:GB_heterogeneity}, with an $L$-smooth and $\mu$-strongly convex global loss $\cL_{\cH}$, for which  algorithm's  output $\hat{\vx}$ can have a non-vacuous guarantee only if $ \frac{f}{n}\le \frac{1}{B^2 + 2}$, and  such that
    \begin{align}
     \label{eq:lb_funcation_gradient_value}
     \cL_{\cH}(\hat{\vx}) - \cL_{\cH}(\vx^*) \ge \frac{G^2}{8\mu} \, \frac{f}{n - (2+B^2)f}, \quad \text{and} \quad 
     \|\nabla \cL_{\cH}(\hat{\vx})\|^2 \ge \frac{\:G^2}{4}\,\frac{f}{n - (2+B^2)f}. 
    \end{align}
\end{theorem}
The  inequality $ \frac{f}{n}\le \frac{1}{B^2 + 2},$ gives an upper bound on the achievable breakdown point.

\subsection{Reduction to an Inexact Gradient Problem}
\label{subsec:reduction}

To protect against adversaries, distributed optimization methods rely on  Algorithm~\ref{alg:reduction}: they robustly estimate the average of the honest gradients using an aggregator $F$, which is, for instance, a geometric median \citep{small1990survey, pillutla2022robust}, a coordinate-wise trimmed mean \citep{yin2018byzantine}, or other methods such as Krum~\citep{blanchard2017machine}. \citet{allouah2023fixing} proposed $(f,\nu)$-robust, a criterion to simultaneously analyze those robust averaging schemes.

\begin{definition}[$(f,\nu)$-robustness \citep{allouah2023fixing}]
\label{def:f_kappa_robustness}
    An aggregation rule $F:\R^{d\times n}\rightarrow \R^d$ is said to be $(f,\nu)$-robust, if $\,\forall \vx_1,\ldots,\vx_n\in \R^d$, and any set $S\subset [n]$ of size $n-f$, 
    \[
    \|F(\vx_1,\ldots,\vx_n) - \overline{\vx}_S\|^2 \le \nu \cdot \frac{1}{|S|}\sum_{i\in S}\|\vx_i - \overline{\vx}_S\|^2,
    \]
    where $\overline{\vx}_S = |S|^{-1}\sum_{i\in S}\vx_i$. We call $\nu$\footnote{We denote it $\nu$ instead of $\kappa$ as in \citet{allouah2023fixing} due to notation incompatibility with the standard notation of the condition number.} the \textit{robustness coefficient} of $F$.
\end{definition}
\begin{algorithm}[t]
  \caption{Inexact Oracles Sampler}
  \label{alg:reduction}
  \begin{algorithmic}
    \STATE {\bfseries Input:} $F$, $\vx$
    \STATE Server sends $\vx$ to clients
    \FOR{client $i\in[n]$ in parallel} 
    \STATE Computes $
        \vg_i = \begin{cases}
            \nabla \cL_i(\vx) \text{ if } i\in \cH\\
            * \text{ if } i \in \cB
        \end{cases}
    $
    \STATE Send $\vg_i$ to Server
    \ENDFOR
    \STATE Server computes $\tilde \nabla \cL_{\cH}(\vx) := F(\vg_1,\ldots, \vg_n)$
  \end{algorithmic}
\end{algorithm}
\begin{table}
  \caption{Robustness coefficients of the coordinate-wise trimmed mean (CWTM), Krum, the geometric median (GM) \citep{allouah2023fixing}, and a multivariate Huber estimator (\cref{app:alg:Huber_Estimator}, a.k.a. CenteredClipping \citep{karimireddy2021learning}).}
  \label{tab:robust_aggregation_values}
  \centering
  \resizebox{\textwidth}{!}{ 
  \begin{tabular}{c c c c c | c}
    \toprule
         Aggregation & Huber & CWTM & Krum & GM &  Lower Bound\\ 
    \midrule
         $\nu$ & $\frac{2f}{n-3f}$ & $\frac{6f}{n-2f}\left(1 + \frac{6f}{n-2f}\right)$ & $6\left(1 + \frac{6f}{n-2f}\right)$ & $4\left(1 + \frac{f}{n-2f}\right)^2$&   $\frac{f}{n-2f}$\\
    \bottomrule
    \end{tabular}
  }
\end{table}

Closed-form robustness coefficients can be found in \cref{tab:robust_aggregation_values}. 
We discuss robustness aggregation techniques in \cref{app:sec:robust_aggregation}, where we also generalize the mixing strategy of \citet{allouah2023fixing}. The strength of \cref{def:f_kappa_robustness} is that under \cref{asmpt:GB_heterogeneity}, aggregating gradients using $(f, \nu)$-robust rules produces an inexact gradient oracle.

\begin{definition}[$(\zeta^2,\alpha)$-inexact oracle]
\label{def:inexact_oracle}
    A $(\zeta^2,\alpha)$-inexact gradient oracle of differentiable function $\cL$, is an oracle $\tilde \nabla \cL$ satisfying, for any $\vx \in \R^d$, $\|\tilde \nabla \cL(\vx) - \nabla \cL(\vx)\|^2 \le \zeta^2 + \alpha \|\nabla \cL(\vx)\|^2$.
\end{definition}

\begin{lemma}
\label{lemma:Byz_to_Inexact_Gradients}
    If $F$ is $(f,\nu)$-robust and local losses are $(G,B)$-heterogeneous then $\tilde{\nabla}\cL_{\cH}(\vx)$, the output of \cref{alg:reduction} is a ($\nu G^2$, $\nu B^2$)-inexact gradient oracle.
\end{lemma}

\subsection{First Consequences of the Inexact Oracle Approach}
\label{subsec:csq_inexact_oracles}
We now exhibit the strength of the inexact gradient formulation by leverage existing results of the literature with the inexact oracle perspective.

\subsubsection{Gradient Descent} 
A first remarkable fact is the \textbf{tightness of our reduction}: plugging \cref{lemma:Byz_to_Inexact_Gradients} into a black-box inexact oracle result~\citep{ajalloeian2020convergence} leads to \cref{cor:GD_inexact}, which %
achieves the best possible asymptotic error.  

\begin{corollary}[\citep{ajalloeian2020convergence}] 
\label{cor:GD_inexact}
Let, $\cL_{\cH}$ be $\mu$-strongly convex and $L$-smooth. Consider the Gradient Descent update \(\vx_{k+1} = \vx_k - \eta \tilde{\nabla} \cL_{\cH}(\vx_k)\), %
where $\tilde{\nabla} \cL_{\cH}(\vx_k)$ stems from \cref{alg:reduction} with a $(f,\nu)$-robust aggregation rule. Then, for $\eta \le 1/L$ and $\nu B^2<1$, it holds
\[
\cL_{\cH}(\vx_K) - \cL_{\cH}(\vx^*)\le \underbrace{\paren{1-\eta \mu\paren{1 - \nu B^2}}^K(\cL_{\cH}(\vx_0)-\cL_{\cH}(\vx^*))\vphantom{\frac{G^2}{(B^2)}}}_{\text{optimization error}} + \underbrace{\frac{\nu G^2}{\mu(1-\nu B^2)}}_{\text{asymptotic error}}.
\]
\end{corollary}

To back the tightness claim, consider a $(f,\nu)$-robust aggregator with $\nu = \frac{c f}{n-2f}$, such as the Huber estimator ($c=3$ for $n>4f$, cf. \cref{tab:robust_aggregation_values}). Then, for $\frac{f}{n} < \frac{1}{2(1+cB^2)}$, \cref{alg:reduction} used in GD converges to the optimal error of \cref{thm:allouha_lower_bound} up to a multiplicative constant, i.e. $O(\frac{f}{n-2f}\frac{G^2}{\mu})$. This result is somehow surprising, as one may think that there would be more structure to the Byzantine problem than simply $(f, \nu)$-robustness, 
for instance by requiring to observe the identities of the clients. Note that this latter point is mandatory for reaching an optimal error when the clients send stochastic i.i.d. gradient oracles (cf.\cref{app:sec:discussion} and \citep{karimireddy2021learning}).

\textbf{Tightness with respect to the dimension.} The definition of $(f,\nu)$-robust aggregator obfuscates the dependency on the dimension in the resulting error term $\nicefrac{G f}{\mu(n-2f)}$, even though a dimension-independent error can be achieved \citep{diakonikolas2019robust, zhu2023byzantine, allouah2025towards}. The Inexact Oracle Approach also tightly encompasses averaging rules with dimension-independent error, as we show in \cref{app:sec:dimension} but omit here for clarity.

\vlong{
\textbf{A new perspective on the optimal error.} We can reinterpret the lower bound in \cref{thm:allouha_lower_bound} as describing the region of $\vx \in \R^d$ on which no information can be leveraged by the server. Indeed, \ref{eq:lb_funcation_gradient_value} gives the gradient norm value below which a $(\zeta^2=\nu G^2, \alpha = \nu B^2)$-inexact oracle resulting from \cref{alg:reduction} may not be informative,
as if $\|\nabla\cL_{\cH}(\vx)\|^2 \le \frac{G^2 \nu}{1 - B^2 \nu}$, then the value $\tilde{\nabla}\cL_{\cH}(\vx) =0$  is a valid $(\zeta^2=\nu G^2, \alpha = \nu B^2)$-inexact oracle, while being obviously un-informative.}

\subsubsection{Acceleration} 
\label{sec:fgm_direct}

A natural approach to reduce the communication complexity, which scales linearly with the number of oracle call using \cref{alg:reduction}, is that of using methods with an \textit{accelerated rate}, for which the oracle complexity to reach a given optimization error level scales linearly with the square root of the condition number $\kappa = L/\mu$, rather than linearly as in GD (cf. \cref{cor:GD_inexact}). The inexact gradient formulation allows to propose such accelerated methods by leveraging previous work on acceleration under $(\zeta^2,\alpha)$-inexact oracles. Indeed, acceleration under inexact oracle have been extensively studied since \cite{schmidt2011convergence,devolder2014first}, and two challenges have risen: achieving an optimal asymptotic error under a nonzero relative error $\zeta^2$, and tolerating the largest relative error $\alpha$ possible.

\begin{table}[h!]
  \centering
  \begin{tabular}{c c c c | c}
        \toprule
         Work & Oracle type & Asymptotic error & Relative error condition & Tolerated $f/n$ \\
         \midrule
         GD$^*$ \citep{ajalloeian2020convergence} & $(\zeta^2,\alpha)$ 
         &\({\zeta^2}/{\mu}\)& \(\alpha \le 1\) & $1/\kappa$ \\
         \midrule
         \citep{devolder2013first} & $(\zeta^2,\alpha=0)$ 
         &\({\sqrt{\kappa} \zeta^2}/{\mu}\)& \(\alpha = 0\) & NA \\
         \citep{vasin2023accelerated} & $(\zeta^2,\alpha)$ 
         & $\sqrt{\kappa}\zeta^2/\mu$ &  $\alpha \le 1/\kappa^{1.5}$ & $1/\kappa^{2.5}$ \\
         \citep{kornilov2025intermediate} & $(\zeta^2=0,\alpha)$ 
         & NA &  $\alpha \le 1/\kappa$& $1/\kappa^{2}$\\
         \citep{vasin2025solving}$^*$ & $(\zeta^2=0,\alpha)$ 
         & NA &  $\alpha \le 1/\kappa$& $1/\kappa^{2}$\\
         \citep{shi2025optimal}$^*$ & $(\zeta^2,\alpha=0)$ 
         &\(\zeta^2/\mu\)& \( \alpha = 0\)& NA \\
         \citep{jambulapati2021robust} & $(\zeta^2,\beta)$ 
         &\(\|\hat\vx - \vx^*\|^2 \le \zeta^2/\mu^2 \)& \(\beta \le 1\) & $1/\kappa^{2}$\\
         \textbf{Ours}$^*$(Cor. \ref{cor:accelerated_nonprox_extragradient})& $(\zeta^2,\alpha)$ 
         & $\zeta^2/\mu$& \( \alpha \le 1/\kappa^{0.5}\) & $1/\kappa^{1.5}$\\
         \bottomrule
    \end{tabular}
    \caption{Asymptotic error and relative error condition of first order methods with $(\zeta^2,\alpha)$ inexact oracles (or $(\zeta^2,\beta)$ cf. \cref{def:inexact_oracle_variation}), either non-accelerated (second row) or with accelerated rate (row $3$ to $9$). Work with a~$^{*}$ are single-loop, the other are restarted. The asymptotic error is measured in function value, i.e.  $\limsup_t \cL_{\cH}(\vx^t) - \cL_{\cH}(\vx^*) \le \rm{asy.~error}$, except in the row of \citep{jambulapati2021robust} which uses the euclidean distance. The last column translates the relative error term in the Byzantine federated setting into a proportion of Byzantine clients tolerated, assuming the usage of a constant-optimal $(f,\nu)$-aggregation rule. This proportion is given using the worst-case upper bounds on $B^2$ in \cref{asmpt:GB_heterogeneity} and $\Delta$ in \cref{asmpt:pairwise_hessian_similarity} of $B^2\le \kappa$ and $\Delta \le L$. Constant factors are omitted.}
    \label{tab:litterature_acceleration}
\end{table}

\textbf{Impact of the absolute error factor $\zeta^2$ on the asymptotic error.} By studying $(\delta,L,\mu)$ oracles, a more generic notion of inexactness, \cite{devolder2013first} showed a lower bound which, translated in the setting of $(\zeta^2,\alpha)$ inexact oracles, suggested that accelerated method with inexact oracles would inevitably suffer from an inflated asymptotic error by a factor $\sqrt{\kappa}$. And, until recently, analysis of accelerated method with absolute error such as \cite{schmidt2011convergence,devolder2013first,vasin2023accelerated} presented such an inflated asymptotic error. Surprisingly, using a restarted optimization scheme that rely on an initial upper bound on the error coupled with projections to gradually bound the step magnitude, \cite{jambulapati2021robust} have shown that, in fact, acceleration and the asymptotic error level of GD could be achieved simultaneously using oracles with an absolute error. Recently, \cite{shi2025optimal} proposed an analysis of the standard single loop Nesterov's acceleration in a federated Byzantine setting, whose proof can be translated into a proof for $(\zeta^2,\alpha=0)$ gradient oracles\footnote{Marginally, their initial analysis of the error term due to the robust aggregation is erroneous, requiring a bounded $\kappa$ to hold. The author developed a new correct analysis, which is not publicly accessible yet.}. Yet, their analysis does not tolerate a relative error term.

\textbf{Impact of the relative error $\alpha$ on the feasibility of acceleration.} It is now understood that acceleration is not fundamentally more challenging than gradient descent under an \emph{absolute} error term. In contrast, the case of a \emph{relative} error term is more ambiguous. Existing analysis either require the relative coefficient to be (at least) smaller than $\kappa^{-1}$ \cite{vasin2023accelerated,kornilov2025intermediate,vasin2025solving}, or they rely on the alternative definition of relative inexactness as \cref{def:inexact_oracle_variation} \cite{jambulapati2021robust}. This asks whether acceleration could be strictly harder than non-accelerated gradient descent under relative inexactness. In \cref{subsec:acceleration} we show that the previous $\alpha \le 1/\kappa$ condition can be improved to $\alpha \le \kappa^{\nicefrac{-1}{2}}$, while enjoying an optimal asymptotic error in presence of an absolute error term.

\section{Gradient Descent under Similarity}
\label{sec:similarity}
We now focus on another approach to reduce the communication complexity, which takes advantage of the similarity between local objectives \citep{shamir2014communication, hendrikx2020statistically, kovalev2022optimal}.

\subsection{Setting}

The key assumption is that the server has cheap (without communication) access to a proxy $\hat{\cL}$ of the global loss function $\cL_{\cH}$ (which requires communication). This proxy may come from the empirical risk computed on a small toy dataset, or from a trusted client that uses its own loss function as the proxy, i.e., $\hat{\cL} = \cL_1$. In particular, the server can solve the following problem at each iteration:
\begin{align}
    \label{eq:prox_update}
    \vx_{k+1} \approx \argmin_{\vx \in \R^d} \Big\{ &\overbrace{\hat{\cL}(\vx)}^{\text{\tiny Proxy Loss}} +  \langle \nabla\overbrace{(\cL_{\cH} - \hat{\cL})}^{\text{\tiny small Hessian}}(\vx_k), \vx \rangle + \frac{1}{2\eta}\|\vx - \vx_k\|^2\Big\}.
\end{align}
This update can be seen from two perspectives, either as a mirror descent step on $\cL_{\cH}$ with step-size $1$ and mirror map $\hat{\cL} + \frac{1}{2\eta}\| \cdot\|^2$~\citep{shamir2014communication,hendrikx2020statistically}, or as a prox-gradient update with the forward step on $\cL_{\cH} - \hat{\cL}$ (smooth) and the backward step on $\hat{\cL}$ (cheap prox)~\citep{kovalev2020optimal, woodworth2023two}. %
In both cases, these manipulations make the problem smoother when $\cL_{\cH}$ and $\hat{\cL}$ are ``similar" enough, speeding up the optimization by allowing significantly larger step sizes. %

\begin{assumption}[Similarity]
\label{asmpt:Hessian_Similarity}
    The difference between the global loss and the proxy $\cL_{\cH}-\hat{\cL}$ is $\Delta$-smooth.%
\end{assumption}

Our result holds under the above assumption, as long as the inexact oracles satisfies an alternative definition of inexact oracles.

\begin{definition}[$(\zeta^2,\beta)$-inexact oracle]
\label{def:inexact_oracle_variation}
    A $(\zeta^2,\beta)$-inexact gradient oracle of a $\mu$-strongly convex function $\cL$, is an oracle $\tilde \nabla \cL$ satisfying, for $\vx^* = \argmin_{\vx} \cL(\vx)$ and any $\vx \in \R^d$, 
    \[\|\tilde \nabla \cL(\vx) - \nabla \cL(\vx)\|^2 \le \zeta^2 + \beta \mu^2 \| \vx - \vx^*\|^2.\]
\end{definition}

\begin{remark}
    A simple analysis shows that, if $\alpha^*$ and $\beta^*$ denote the smallest $\alpha$ and $\beta$ achieving respectfully \cref{def:inexact_oracle} and \cref{def:inexact_oracle_variation} for a given inexact gradient oracle, then it holds that $\alpha^* \le \beta^* \le \kappa^2\alpha^*$.
\end{remark}

When the clients satisfy the pairwise Hessian Similarity assumption, on the one hand \cref{def:inexact_oracle_variation} arises from using as proxy a trusted client's loss, $\hat{\cL} = \cL_1$, and on the other hand using \cref{alg:reduction} with an $(f,\nu)$-aggregator yields a $(\zeta^2,\beta)$-inexact oracle \ref{def:inexact_oracle_variation}.

\begin{assumption}[Pairwise Hessian Similarity]
\label{asmpt:pairwise_hessian_similarity}
    All local honest loss functions $(\cL_i)_{i \in \cH}$ are twice differentiable, and satisfy $\forall i,j \in\cH,\forall \vx \in \R^d, \quad \|\nabla^2 \cL_{i}(\vx) - \nabla^2 \cL_{j}(\vx)\|_{op} \le \Delta$.
\end{assumption}

\begin{proposition}
    \label{prop:similarity_to_GB_heter}
    Let $(\cL_i)_{i\in\cH}$ be convex functions. Then, \cref{asmpt:pairwise_hessian_similarity} ensures that
    \begin{equation}
    \frac{1}{|\cH|}\sum_{i\in\cH}\|\nabla \cL_i(\vx)-\nabla \cL_{\cH}(\vx)\|^2 \le\frac{2}{|\cH|}\sum_{i\in\cH}\|\nabla \cL_i(\vx^*)\|^2 + 2\Delta^2 \|\vx - \vx^*\|^2.
    \end{equation}
    Which, if additionally $\cL_{\cH}$ is $\mu$-strongly convex, implies $(G,B)$-Heterogeneity with 
    \begin{equation}
    \label{eq:def_G_B_from_similarity}
    G^2 = \frac{2}{|\cH|}\sum_{i\in\cH}\|\nabla \cL_i(\vx^*)\|^2 \text{ and } B^2 = \frac{2\Delta^2}{\mu^2}.
    \end{equation}
\end{proposition}

We refer the reader to \cref{app:proof:similarity_to_GB_heter} for the proof. We thus obtain the following corollary.%

\begin{corollary}
\label{lemma:similarity_and_Byz_to_Inexact_Gradients}
    If $F$ is $(f,\nu)$-robust and local losses $(\cL_i)_{i\in\cH}$ are convex, satisfy \cref{asmpt:pairwise_hessian_similarity}, and $\cL_{\cH}$ is $\mu$-strongly convex, then $\tilde{\nabla}\cL_{\cH}(\vx)$ the output of \cref{alg:reduction} is a ($\zeta^2$, $\beta$)-inexact gradient oracle, with $\zeta^2=\nu G^2$, $\beta =\nu B^2$ and ($G^2$,$B^2$) as in \cref{eq:def_G_B_from_similarity}.
\end{corollary}

\subsection{Convergence Results}
 We now state the complete Byzantine-robust distributed algorithm and its convergence guarantees.
 
\begin{algorithm}[h]
  \caption{Preconditioned Inexact Gradient with Similarity}
  \label{alg:FedProxyProx}
  \begin{algorithmic}
    \STATE {\bfseries Input:} $F$, $\eta$, $\vx_0$, $\hat{\cL}$
     \FOR{$k=0$ {\bfseries to} $K-1$}
     \STATE Sample $\tilde\nabla \cL_{\cH}(\vx_k)$ using \cref{alg:reduction}.
    \STATE $\vx_{k+1} \approx \argmin_{\vx \in \R^d} \left[\phi^k_{\eta}(\vx) =  \langle \tilde{\nabla} \cL_{\cH}(\vx_k) - \nabla \hat{\cL}(\vx_k), \vx \rangle + \frac{1}{2\eta}\|\vx - \vx_k\|^2 + \hat{\cL}(\vx)\right]$.
    \ENDFOR
  \end{algorithmic}
\end{algorithm}

\begin{theorem}
\label{thm:proxyprox_conv}
    Assume $\cL_{\cH}$ is $\mu$-strongly convex and $L$-smooth, and $(\cL_{\cH}-\hat{\cL})$ is $\Delta$-smooth. Consider \cref{alg:FedProxyProx} with a $(\zeta^2,\beta)$ inexact oracle (\cref{def:inexact_oracle_variation}) with $\beta < \nicefrac{1}{4}$, a step-size $\eta \le \nicefrac{1}{2\Delta}$, and where the proximal step is computed with precision
    \begin{equation}
    \label{eq:precision_prox}
    \|\nabla \phi^k_{\eta}(\vx_{k+1})\|^2 \le \frac{\mu\Delta}{4}\|\vx_{k+1} - \vx_k\|^2 + E^2.
    \end{equation}
    Define $s_k = (1 + \frac{\eta \mu}{3})^{k}$, and $S_K = \sum_{k=0}^K s_k$. The weighted average $\hat{\vx}_k := \frac{1}{S_K}\sum_{k=0}^K s_k \vx_k$ satisfies
    \begin{align}
    \cL_{\cH}(\hat\vx_{K}) - \cL_{\cH}(\vx^*) \le & \left(1 + \frac{\eta \mu}{3}\right)^{1-K}\frac{2}{\eta(1-4\beta)}\|\vx_0 - \vx^*\|^2 + 2\frac{\zeta^2 + E^2}{\mu(1-4\beta)}.
    \end{align}
\end{theorem}
Importantly, in \cref{alg:FedProxyProx}, each $\vx_{k+1}$ is obtained by solving an optimization subroutine on the loss $\phi_\eta^k$ which the server can easily access. \cref{eq:precision_prox} captures the error of this subroutine, and $E^2$ can be made arbitrarily small by allocating more computational budget on the server side, without involving more communication or requiring more homogeneity or honest nodes. This subproblem can also be warm-started, starting the minimization from $\vx_k$, significantly reducing the computational burden. Interestingly, because of the absolute error due to adversarial nodes, taking $E^2 = o(\zeta^2)$ can save significant computational resources without affecting the precision. 

Our proof of \cref{thm:proxyprox_conv}, available in \cref{app:proof:proxyprox}, draws from the one of the one \citet{woodworth2023two} which covers the method with approximate proximal steps, but exact gradients. To handle the relative inexactness we introduce a new Lyapunov function. \\
We now instantiate \cref{thm:proxyprox_conv} in the Byzantine setting.

\begin{corollary}
\label{cor:pigs_byzantin}
    Under the assumption of \cref{lemma:similarity_and_Byz_to_Inexact_Gradients}, consider \cref{alg:FedProxyProx} with $\hat{\cL} =\cL_1$ and $\tilde{\nabla}\cL_{\cH}$ computed with \cref{alg:reduction} and a $(f,\nu=\cO(\nicefrac{f}{n-2f}))$ robust-aggregation rule. Take $\eta = \nicefrac{1}{2\Delta}$, and $E^2\le O(\frac{f}{n-2f}G^2)$. Then, under $\frac{f}{n-2f}B^2 = O(1)$, we have after $K$ iterations that\[
    K=\mathcal{O}\left(\frac{\Delta}{\mu}\log(1/\eps)\right)\implies \cL_{\cH}(\hat{\vx}_K) - \cL_{\cH}^* = \mathcal{O}\left(\frac{f}{n-2f}\frac{G^2}{\mu} + \eps\right).
    \]

\end{corollary}

\begin{remark}
     Note that $\nicefrac{\Delta}{\mu}$ can be much smaller than $\kappa = \nicefrac{L}{\mu}$, the condition number of the global loss $\cL_{\cH}$, thus significantly speeding-up optimization. Indeed, assume that $\cL_i(\vx) := \frac{1}{|\cD_i|}\sum_{a \in \cD_i}l(a,\vx) + \lambda/2\|\vx\|^2$, where each dataset $\cD_i$ consist of $m$ i.i.d. samples from the same distribution $\cD$. Assume $\|\nabla^2_{\vx}l(a,\vx)\|_{op} \le L$. %
     Then w.l.o.g., $\|\nabla^2 \cL_{\cH}(\vx)\|_{op} \le L$, while Hoeffding's inequality for matrices~\citep{tropp2015introduction} yields, with  probability at least $1-\delta$, %
    \[
    \Delta = \|\nabla^2 \cL_1(\vx) - \nabla^2 \cL_{\cH}(\vx)\|_{op} = \mathcal{O}\left(\sqrt{\frac{L^2\log(d/\delta)}{m}}\right).
    \]
    It follows $\nicefrac{\Delta}{\mu} \le \mathcal{O}\left(\nicefrac{\kappa}{\sqrt{m}}\,\log(d/\delta)\right)$. Thus, the iteration complexity of \cref{alg:FedProxyProx} to reach a given precision can be reduced by up to $\tilde \cO(\sqrt{m})$ compared to Byzantine-robust D-GD. When local datasets are large, this yields a huge improvement in the communication complexity. For refined and stronger control of $\Delta$, including better dependencies on $m$, see \citet{hendrikx2020statistically}.
\end{remark}

\section{An Accelerated Inexact Extra-gradient Method}
\label{sec:extragradient_accelerated}
In this section, we adapt the accelerated method from \citet{kovalev2022optimal} to the inexact oracle setting, leading to \cref{alg:acc_inexact_extragradient}, an accelerated method for distributed optimization under similarity. 

\subsection{General Inexact Algorithm}
We first give the convergence guarantees for \cref{alg:acc_inexact_extragradient}, independently of the Byzantine robust application. 
\begin{theorem}
    \label{thm:inexact_extragradient_convergence}
Assume that $\cL_{\cH}$ is $\mu$-strongly convex and $L$-smooth, that $\cL_{\cH}-\hat{\cL}$ is $\Delta$-smooth, %
and that $\tilde{\nabla}\cL_{\cH}$ is a $(\zeta^2,\alpha)$-inexact oracle with $\alpha < \min\left(\frac{\Delta^2}{4L^2},\frac{1}{16}\right)$. Consider \cref{alg:acc_inexact_extragradient} with the parametrization 
\[
\gamma = \mu,\quad \theta = \frac{1}{2\sqrt{2}\Delta},\quad \eta = \min\left(\frac{1}{2\mu}, \frac{\theta}{24 \tau}\right), \quad \text{and} \quad \tau = \min\left(\frac{\theta \mu}{16 \alpha}, \sqrt{\frac{\theta \mu}{48}}\right).
\]
Then, if \(
\sqn{\nabla \phi_\theta^k(\vx_f^{k+1})} \le  \frac{1}{8\theta^2}\sqn{\vx_f^{k+1} - \vx_g^k},
\) we have that 
\[
\cL_{\cH}(\vx_f^K) - \cL_{\cH}(\vx^*) = O\Big((1 - \rho)^K R + \frac{\zeta^2}{\mu}\Big),
\]
where $\rho = c\min\left(\sqrt{\frac{\mu}{\Delta}}, \frac{\mu}{\alpha \Delta}\right)$ with $c >0$ a constant, and $R = \tau^2\Delta\|\vx^0 - \vx^*\|^2 + \cL_{\cH}(\vx^0) - \cL_{\cH}(\vx^*)$.
\end{theorem}
We refer the reader to \cref{app:proof:extragradient} for the proof of \cref{thm:inexact_extragradient_convergence}.

\begin{algorithm}[h]
\caption{Accelerated Inexact Extra-gradient}
\label{alg:acc_inexact_extragradient}
	\begin{algorithmic}[1]
		\STATE {\textbf Input:} $\vx^0=\vx_f^0 \in \R^d$, inexact oracle sampler $\tilde{\nabla}\cL_{\cH}$.
		\STATE {\textbf Parameters:} $\tau \in (0,1)$, $\eta,\theta,\gamma > 0, K \in \{1,2,\ldots\}$
		\FOR{$k=0,1,2,\ldots, K-1$}
			\STATE $\vx_g^k = \tau \vx^k + (1-\tau) \vx_f^k$\label{ae:line:1}
			\STATE \resizebox{\linewidth}{!}{$\vx_f^{k+1} \approx \argmin_{\vx \in \R^d}\left[ \phi^k_{\theta}(\vx):=  \langle \tilde{\nabla} \cL_{\cH}(\vx_g^k) -  \nabla \hat{\cL}(\vx_g^k),\vx - \vx_g^k\rangle + \frac{1}{2\theta}\sqn{\vx - \vx_g^k} + \hat{\cL}(\vx)\right]$}\label{ae:line:2}%
			\STATE $\vx^{k+1} = \vx^k + \eta\gamma (\vx_f^{k+1}  - \vx^k)- \eta \tilde{\nabla} \cL_{\cH}(\vx_f^{k+1})$\label{ae:line:3} %
		\ENDFOR
		\STATE {\bf Output:} $\vx^K$
	\end{algorithmic}
\end{algorithm}

\subsection{Acceleration Without Similarity} 
\label{subsec:acceleration}
When no cheap proxy is available, one can take $\hat{\cL} = 0$. This removes the proximal update from line 5 of \cref{alg:acc_inexact_extragradient}, leading to
\[
\vx_f^{k+1} = \vx_g^{k} - \theta \tilde{\nabla}\cL_{\cH}(\vx_g^k).
\]
This provides a forward-only inexact first-order accelerated algorithm, with state-of-the-art guarantees.
\begin{corollary}[Acceleration with inexact oracles]
\label{cor:accelerated_nonprox_extragradient}
    Let the setting of \cref{thm:inexact_extragradient_convergence} holds, with $\hat{\cL}=0$, and thus $\Delta=L$. Then, if $\alpha < \frac{1}{\sqrt{\kappa}}$, \cref{alg:acc_inexact_extragradient} ensures\[K=O(\sqrt{\kappa}\log(\nicefrac{R}{\eps})) \implies 
    \cL_{\cH}(\vx_f^{K}) - \cL_{\cH}(\vx^*) = O\left(\nicefrac{\zeta^2}{\mu} + \eps\right).\]

\end{corollary}
In this regime, \cref{cor:accelerated_nonprox_extragradient} is the first convergence result of an accelerated algorithm for $(\zeta^2,\alpha)$-inexact oracles with both $\zeta^2\neq 0$ and $\alpha\neq 0$, while enjoying an optimal asymptotic error. Additionally, as highlighted in \cref{tab:litterature_acceleration}, the relative error condition improves over previous work by a factor $\sqrt{\kappa}$, as we require $\alpha < 1/\sqrt{\kappa}$ instead of \(\alpha < 1/\kappa\).

\subsection{Focus on the Byzantine Setting}
\label{subsec:focus_extragrad_byz}

We now instantiate the convergence result of \cref{alg:acc_inexact_extragradient} in the Byzantine setting, first in a standard distributed setting, without relying on a proxy loss $\hat{\cL}$, and then in the similarity setting.

\begin{corollary}[Byzantine-resilient plain acceleration]
    Let $\cL_{\cH}$ be $L$-smooth and $\mu$-strongly convex, let \cref{asmpt:GB_heterogeneity} hold and assume $B^2 \frac{f}{n-2f} \in O(\frac{1}{\sqrt{\kappa}})$. Then \cref{alg:acc_inexact_extragradient} with an $(f,\nu=O(\nicefrac{f}{n-2f}))$-aggregator, $\hat{\cL}=0$, after $K$ iteration ensures\[
    K = O\Big(\sqrt{\kappa}\log(1/\eps)\Big) \implies \cL_{\cH}(\vx_f^K) - \cL_{\cH}(\vx^*) = O\Big(\frac{f}{n-2f}\frac{G^2}{\mu} + \eps \Big).
    \]
\end{corollary}

\begin{corollary}[Byzantine-resilient acceleration under similarity]
\label{cor:extragrad_byz}
    Under the assumptions of \cref{cor:pigs_byzantin}, if additionally $\frac{f}{n-2f} \in O(\frac{1}{\kappa^2})$, then \cref{alg:acc_inexact_extragradient} with an $(f,\nu=O(\nicefrac{f}{n-2f}))$-aggregator after $K$ iteration ensures\[
    K=O\Big(\sqrt{\nicefrac{\Delta}{\mu}}\log(1/\eps)\Big) \implies \cL_{\cH}(\vx_f^K) - \cL_{\cH}(\vx^*) = O\Big(\frac{f}{n-2f}\frac{G^2}{\mu} + \eps \Big).
    \]
\end{corollary}

\cref{cor:extragrad_byz} is the accelerated version of \cref{cor:pigs_byzantin} (the PIGS result). Again, leveraging optimization under similarity replaces the standard dependence of the number of iterations on $\kappa = \nicefrac{L}{\mu}$ with $\nicefrac{\Delta}{\mu}$, which is smaller in general.

Yet, while the robustness regime of PIGS is $\frac{f}{n-2f} \le O(\frac{1}{B^2})= O(\frac{\mu^2}{\Delta^2})$ (when a robust aggregator with $\nu = O(\nicefrac{f}{n-2f})$ is used), which does not depends on the smoothness $L$ at all, \cref{thm:inexact_extragradient_convergence} requires $\frac{f}{n-2f} \leq \Omega(\kappa^{-2})$, which is much larger if we expect optimization under similarity to be worth the extra computational effort. We recover the same behavior as with standard acceleration: acceleration speeds up convergence, but can only be used if the number of Byzantine nodes is low enough (smaller breakdown point).

\section{Experiments}
\label{sec:experiments}

\begin{wrapfigure}{r}{0.5\textwidth}
    \vspace{-12pt}
    \centering
    \includegraphics[width=0.49\textwidth]{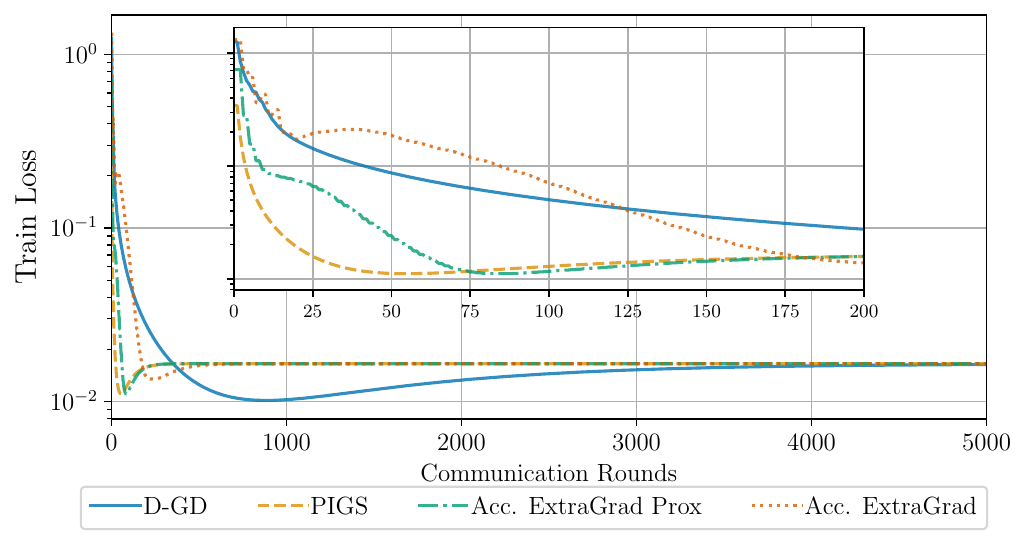}
    \caption{Distributed Logistic Regression, with $40$ honest and $1$ Byzantine client in the mildly heterogeneous setting ($a=5$), with a $l2$-regularization $\mu = 10^{-2}$. Train loss $\cL_{\cH}(\vx_k) - \cL_{\cH}^*$ under ALIE attack. Plot with a zoom on the first 200 iterations.}
    \label{plot:logistic_heterogeneous_loss}
\end{wrapfigure}

\subsection{Setup}

\textbf{Setup.} We consider a l2-regularized Logistic Regression task on the MNIST dataset \citep{lecun1998gradient}. To simulate heterogeneity among clients, we follow the Dirichlet partitioning method from \citet{hsu2019measuring, allouah2023fixing}: each client’s dataset class proportions are sampled from a Dirichlet distribution with parameter $a$. Lower $a$ corresponds to higher heterogeneity. We refer to $a=1$ as \textit{highly heterogeneous} and to $a = 5$ as \textit{mildly heterogeneous}. Due to space constraints, we refer the reader to \cref{app:sec:expe_sweep_heter_proportion} for experiments with a sweep over both the proportion of adversaries and the heterogeneity level.

\textbf{Attacks.} The attacks simulated are the Label Flipping attack \citep{allenbyzantine}, a Gaussian noise attack, as well as variants of A Little Is Enough (ALIE) \citep{baruch2019little} and Inner Product Manipulation \citep{xie2020fall}. The scale of attacks is fixed using a line search to maximize impact. We proceed with an extensive sweep over attacks and defenses in \cref{app:exp:attack_defenses}. %

\textbf{Optimization Algorithms.} We compare the accelerated extra-gradient with similarity (Acc. Extragrad Prox) and without (Acc. Extragrad), Distributed Gradient Descent (D-GD), and PIGS. We tune optimizer parameters (in particular the learning rate) via hyperparameter search in the absence of Byzantine clients. All experiments - except from \cref{app:exp:attack_defenses} - are conducted using the Huber aggregation (i.e. \cref{app:alg:Huber_Estimator}) in the reduction (i.e. \cref{alg:reduction}). In \cref{app:exp:attack_defenses} specifically, we also compare the Huber aggregation with Trimmed Mean \citep{yin2018byzantine}, Krum \citep{blanchard2017machine}, and Covariance-agnostic Filter (CAF) \citep{allouah2025towards}.

Our code, accessible \href{https://github.com/renaudgaucher/Inexact-gradient-acceleration-and-similarity}{here}, is built on top of the ByzFL library \citep{gonzalez2025byzflresearchframeworkrobust}.

\subsection{Analysis}

As shown in all experiments, D-GD is significantly slower than the accelerated and/or similarity-based approaches. In particular, it rises from \cref{plot:logistic_heterogeneous_loss} that, as predicted by theory, the train loss $\cL_{\cH}(\vx^t)-\cL_{\cH}(\vx^*)$ does not converge to zero, but to a non-zero error level. We also see from  \cref{app:tab:dir01,app:tab:dir10,app:tab:dir50,app:tab:iid} that 1)  the error level rises with the proportion of adversaries and the heterogeneity, and that more heterogeneous datasets lead to algorithm failure for a smaller proportion of Byzantine clients. 
2) the robustness of PIGS and D-GD are similar, and superior to that of the extra-gradient methods. 
3) When comparing PIGS and the proximal extragradient method, it appears that PIGS is faster in the i.i.d. environment, while in the highly heterogeneous environment the proximal extragradient method is faster. We interpret this as a constant overhead in the accelerated method, which is non-negligible if $\Delta/\mu$ is small, i.e. in homogeneous settings, while the usage of acceleration has more impact for large $\Delta/\mu$, i.e. in heterogeneous environments. Yet, in such settings, the proximal extra-gradient method tolerates significantly fewer Byzantine clients than PIGS.

\vshort{
}

\section{Conclusion}
In this work, we proposed to abstract away the problem of Byzantine Optimization by formalizing it as an inexact gradient oracle problem. Our approach opens up a new systematic approach to study and design Byzantine-resilient algorithms. We leverage this to propose both a non-accelerated optimization-under-similarity method and an accelerated extra-gradient scheme, which can be used both in the standard distributed setting and in the similarity setting. Those methods have better communication complexity (and a better computation complexity for the accelerated method without similarity), as we show theoretically and experimentally.

\section*{Acknowledgement}
We would like to thank Adrien Taylor for the helpful discussions and feedback.
The work of Aymeric Dieuleveut and Renaud Gaucher was supported by French State aid managed by the Agence Nationale de la Recherche (ANR) under the France 2030 program with the reference ANR-23-PEIA-005 (REDEEM project). The work of Aymeric Dieuleveut was also supported by Hi!Paris - FLAG chair.

%% file: appendix.tex
\section{Robust Aggregation Rules}
\label{app:sec:robust_aggregation}

In this section, we recall $(f,\nu)$-robust averaging guarantees of standard aggregators. We also add two contributions to the vast literature of robust aggregators:

\begin{enumerate}
    \item \textbf{A new analysis of the multivariate Huber estimator.} In \cref{app:sec:huber}, we discuss the relation with the multivariate estimator and the centered clipping scheme \citep{karimireddy2020byzantine}. We prove the $(f,\nu)$-robustness of a multivariate Huber estimator, which remarkably has state-of-the-art constant factors on $\nu$.
    \item \textbf{A generalization of the Nearest Neighbor Mixing (NNM) strategy of \citet{allouah2023fixing}.} \citet{allouah2023fixing} show that relying on a mixing strategy \textit{before} robustly averaging the input vectors allows for better robustness guarantees. We generalize this strategy with the notion of \textit{robust mixing}, and propose a pre-aggregator with better theoretical guarantees. 
\end{enumerate}

\paragraph{Overview of the $(f,\nu)$-robustness.} In \cref{app:tab:robust_aggregation_values} we recall the robustness guarantees \citep{allouah2023fixing} of the coordinate-wise trimmed mean, Krum, the geometric median and the coordinate-wise median. We additionally provide the robustness guarantees of our analysis of the multivariate Huber estimator. 

\begin{table}[h]
  \caption{Robustness coefficients of the coordinate-wise trimmed mean (CWTM), Krum, the geometric median (GM), the coordinate-wise median (CWM), and a multivariate Huber estimator.}
  \label{app:tab:robust_aggregation_values}
  \centering
 \resizebox{\linewidth}{!}{\begin{tabular}{c c c c c c | c}
    \toprule
         Aggregation & Huber & CWTM & Krum & GM & CWM &  Lower Bound\\ 
    \midrule
         $\nu$ & $\frac{2f}{n-3f}$ & $\frac{6f}{n-2f}\left(1 + \frac{6f}{n-2f}\right)$ & $6\left(1 + \frac{6f}{n-2f}\right)$ & $4\left(1 + \frac{f}{n-2f}\right)^2$ & $4\left(1 + \frac{f}{n-2f}\right)^2$&   $\frac{f}{n-2f}$\\
    \bottomrule
    \end{tabular}}
\end{table}

\textbf{Generic definition of a robust mixing.} \citet{allouah2023fixing} suggest improving aggregation rules by pre-processing the input values through a pre-aggregation step, and specifically the Nearest Neighbors Mixing (NNM) strategy. We generalize this mixing strategy in the following formal definition:
\begin{definition}
    A map $F_1:\R^{n\times d} \rightarrow \R^{n\times d}$ is said to be a $(f, \delta)$ robust mixing if, when $|\cB| \le f$, and given $F_2$ a $(f,\nu)$-robust aggregation rule, $F_2 \circ F_1$ is a $(f, \tilde \nu)$-robust aggregation rule with
    \[
    \tilde \nu = \delta (1 + \nu).
    \]
\end{definition}

We list in \cref{app:tab:robust_pre_aggregation_values} a few robust mixing strategies, which can be built generally using the robust-gossip framework, as we show in \cref{app:sec:robst_gossip_to_robust_mixing}.

\begin{table}[h]
  \caption{Robust Pre-Aggregation / Mixing Strategies. F in F-RG is a $(\rho, b=\frac{f}{n-f})$-robust summand. }
  \label{app:tab:robust_pre_aggregation_values}
  \centering
  \begin{tabular}{c c c c | c}
        \toprule
         Pre-Aggregation & CS$_+$-RG & NNM/NNA & F-RG  & Lower Bound\\
         \midrule
         $\delta$ &   $\frac{4f}{n-f}$ & $\frac{8f}{n-f}$  & $\frac{2\rho f}{n-f}$&  \\
         Breakdown & $\frac{1}{5}$ & $\frac{1}{9}$ & $\frac{1}{2\rho + 1}$ &  $\frac{1}{3}$ \\
         \bottomrule
    \end{tabular}
\end{table}

\subsection{Robustness of the Multivariate Huber Estimator.}

\label{app:sec:huber}

In this section, we propose a multivariate Huber estimator \citep{huber1964robust} with a dynamic thresholding rule, and show that it is a $(f,\nu)$-robust averaging rule.

The multivariate Huber robust estimator, relating back to the very beginning of robust statistics, is defined as 
\[
F_{Huber}(\vx_1,\ldots,\vx_n; \tau) = \argmin_{\vx} \left\{\Psi(\vx; (\vx_i)_{i\in [n]}) := \sum_{i\in [n]} \psi_{\tau}(\vx_i - \vx) \right\}, 
\]
where $\tau>0$ is a constant threshold and $\psi_\tau$ denotes the Huber loss defined as 
\begin{equation*}
\psi_\tau(\vx) = 
    \begin{cases}
        \frac{1}{2}\|\vx\|^2 \text{ if } \|\vx\| \le \tau,\\
        \tau\|\vx\| - \frac{\tau}{2}\text{ if } \|\vx\| > \tau.
    \end{cases}
\end{equation*}

\begin{remark}
    The centered-clipping scheme (CC) \citep{karimireddy2020byzantine} is an approximation of the multivariate Huber estimator with constant threshold coincide. \new{Indeed, the centered clipping aggregation rule starts from an estimate of the average $\vy^0$, and iteratively improve this estimate during $T$ iterations using}
\[
\new{\vy^{t+1} = \vy^t - \eta \sum_{i \in [n]}\mathrm{clip}(\vx_i^t - \vy^t;\tau)},\quad \text{where}\quad \clip(\vx;\tau) := \frac{\vx}{\|\vx\|} \min(\tau, \|\vx\|).
\]
Which corresponds to gradient descent steps on the Huber empirical loss function $\Psi(\,\cdot\,; (\vx_i)_{i\in[n]})$.

Note however that choice of the threshold $\tau$ in the multivariate Huber estimator (and CC) impacts the robustness of the procedure: indeed, choosing $\tau \rightarrow \infty$ yields a (non-robust) simple average, while $\tau \rightarrow 0$ outputs $\vy_0$, i.e. the initial guess of the average. 
\end{remark}

In practice, the clipping threshold used in the theoretical guarantees in \citet{karimireddy2021learning} can not be computed, as it depends on the variance of the honest inputs $(\vx_i)_{i \in \cH}$. To alleviate this flaw, we propose in \cref{app:alg:Huber_Estimator} a multivariate Huber estimator with a threshold choice agnostic to the (unknown) variance of the honest inputs. The estimator iteratively updates the threshold $\tau$, such that the multivariate Huber estimator is, after a few iterations, $(f,\nu)$-robust.

\begin{algorithm}[h]
  \caption{Byzantine-Robust Huber Estimator}
  \label{app:alg:Huber_Estimator}
  \begin{algorithmic}
    \STATE {\bfseries Input:} Initial average guess $\vy_0$, Upper bound on the number of adversaries $f$, Input vectors $(\vx_i)_{i \in [n]}$
     \FOR{$t=0$ {\bfseries to} $T-1$}
     \STATE Sort $(\|\vy^\new{t} - \vx_i\|, i\in[n])$ with a permutation $\sigma_k:[n] \rightarrow [n]$ into 
     \[
     \|\vy^\new{t} - \vx_{\sigma_k(1)}\| \ge \ldots \ge \|\vy^\new{t} - \vx_{\sigma_k(n)}\|
     \]
     \STATE Set $\tau^\new{t} = \|\vy^\new{t} - \vx_{\sigma_k(2f)}\|$
     \STATE Update $\new{\vy^{t+1}} = \vy_\new{t} - \frac{1}{n-f}\sum_{i\in[n]}\clip(\vy^\new{t} - \vx_i; \tau^\new{t})$
    \ENDFOR
    \STATE Output $\vy^T$
  \end{algorithmic}
\end{algorithm}

The aggregation rule given by \cref{app:alg:Huber_Estimator} is probably robust.

\begin{proposition}
    Consider $\nu_0 = \frac{2f}{n-f}$, and $\nu = \frac{2f}{n-3f}$, then for any $\vx_1,\ldots,\vx_n \in \R^d$, and any $S\subset[n]$ of size $n-f$, \new{after $T$ iterations} \cref{app:alg:Huber_Estimator} yields
    \[
    \|\vy^{\new{T}}-\overline{\vx}_{S}\|^2 \le \nu_0^{\new{T}}\|\vy_{0}-\overline{\vx}_{S}\|^2 + \nu (1 - \nu_0^T)\frac{1}{n-f}\sum_{i \in S}\|\vx_i - \overline{\vx}_{S}\|^2.
    \]
    Thus, the multivariate Huber estimator is $(f,\nu+\eps)$-robust after 
    \[\frac{\log(\|\vy_0 - \overline{\vx}_{S}\|^2 -  \frac{\nu}{n-f}\sum_{i\in S}\|\vx_i - \overline{\vx}_S\|^2) + \log(1/\eps)}{\log(1/\nu_0)}\]
    iterations. 
\end{proposition}

\begin{proof}
    Writing the update of \cref{app:alg:Huber_Estimator} and applying the triangle inequality yields
    \begin{align*}
        \|\vy^\new{t} - \overline{\vx}_{S}\|^2 &= \left\|\frac{1}{n-f}\sum_{i \in S} \vy^\new{t} - \vx_i - \clip(\vy^\new{t} - \vx_i;\tau^k) + \frac{1}{n-f}\sum_{i \in [n]\backslash S} \clip(\vx_i - \vy^\new{t}; \tau^\new{t})\right\|^2\\
        &\le \left(\frac{1}{n-f}\sum_{i \in S} (\|\vy^\new{t} - \vx_i\| - \tau^\new{t})_+ + \frac{f}{n-f}\tau^\new{t}\right)^2,
    \end{align*}
    where we denoted $(a)_+ = \max(a,0)$.

    Setting $\tau^\new{t}$ such that they are at least $f$ terms in the left sum, and at most $2f$, for instance $\tau^\new{t} = \|\vy^\new{t} - \vx_{\sigma(2f)}\|$, yields using Cauchy-Schwarz inequality,
    \begin{align*}
        \|\vy^\new{t} - \overline{\vx}_{S}\|^2 
        &
        \new{\le  \left(\frac{1}{n-f}\sum_{i \in S} \|\vy^\new{t} - \vx_i\|1_{\sigma_k(i)\le 2f} \right)^2}\\
        &\le \frac{2f}{n-f}\frac{1}{n-f}\sum_{i \in S}\|\vy^\new{t} - \vx_i\|^2.
    \end{align*}

    The latter inequality can be decomposed, denoting $d^\new{t} = \|\vy^\new{t} - \overline{\vx}_{S}\|^2$ and $e = \frac{1}{n-f}\sum_{i \in S}\|\vx_i - \overline{\vx}_S\|^2$, as 
    \[
    d^{t+1} \le \nu_0 (d^\new{t} + e).
    \]
    This is an algebraic-arithmetic-like sequence. Consequently, it holds
    \[
    \left(d^\new{t} - \frac{\nu_0}{1 - \nu_0 }e\right) \le \nu_0^t \left(d_0 - \frac{\nu_0 }{1-\nu_0}e\right).
    \]
    From which the proposition derives, using $\nu = \frac{\nu_0}{1-\nu_0}$.
\end{proof}
\begin{remark}
    To our knowledge, the multivariate Huber estimator is currently the $(f,\nu)$-robust estimator with the tightest $\nu$. For instance, investigating the range of $f\in [n]$ for which the condition $\nu < 1$ holds - which is required to have useful convergence guarantees when plugged in \cref{alg:reduction} - yields $5f < n$ for the Huber estimator, while for instance CWTM requires approximately $25f < n$.

    Similarly, the best $(f,\nu)$ robust averaging rule deriving from \cref{app:prop:pre_processing_robust_avg} is $\nu = \frac{2\rho f}{n - 2(1+\rho)f}$, which, when considering that necessarily $\rho > 1$, is always worse than $\frac{2f}{n-3f}$.
\end{remark}

\subsection{From Byzantine-Robust Gossip to Byzantine-Robust Distributed Learning}
\label{app:sec:robst_gossip_to_robust_mixing}
\begin{algorithm}[h]
  \caption{Robust-Gossip based Pre-Aggregation}
  \label{app:alg:F_RG_preaggregation}
  \begin{algorithmic}
    \STATE {\bfseries Input:} $F_2$ $(\rho, \frac{f}{n-f})$-robust summand rule, $F_1$ $(f,\nu)$-robust aggregation rule, vectors input $(\vx_i)_{i\in [n]}$.
    \FOR{Client $i \in [n]$}
     \STATE Client Send $\vx_i$ to Server
    \ENDFOR
    \STATE Server Update $\vy_i = \vx_i - F( (\vx_i - \vx_j, \frac{1}{n-f})_{j\in [n]})$
  \end{algorithmic}
\end{algorithm}

Nearest Neighbor Mixing (NNM) can be seen as an emulation of the Nearest Neighbor Averaging (NNA) algorithm \citep{farhadkhani2023robust}. As such, the mixing strategy constitutes an intrinsically decentralized aggregation scheme, which can be further generalized within the Robust Gossip framework \citep{gaucher2025unified}. 
More precisely, if the server simulates a single iteration of an $F$-Robust Gossip \citep{gaucher2025unified}, with $F$ a $(b,\rho)$-robust summand, the resulting update satisfies the mixing property required by NNM. This allows us to design even tighter pre-aggregation strategy.

$F$-Robust Gossip (or $F$-RG) is a decentralized algorithm, in the sense that all clients are seen as vertices in a communication graph $\cG$, in which they communicate synchronously with their neighbors through the edges. The algorithm relies on so-called $(b,\rho)$-robust summand functions, defined as follows:

\begin{definition}[$(b,\rho)$--robust summation] Let $b, \rho \ge 0$. An aggregation rule $F:(\R_+ \times \R^d)^n \rightarrow \R^d$ is a \emph{$(b, \rho)$--robust summation} if, for any vectors $(\vz_i)_{i \in [n]}\in (\R^d)^n$, any weights $(\omega_i)_{i \in [n]} \in \R_+^n$ and any $S \subset [n]$ such that $\sum_{i \in \overline{S}} \omega_i \le b$ (where $\overline{S} := [n]\backslash S$),
\[
\bigg\|F\big((\omega_i, \vz_i)_{i \in [n]}\big) - \sum_{i\in S} \omega_i \vz_i\bigg\|^2 \le \rho b \sum_{i\in S} \omega_i \|\vz_i\|^2.
\]
\end{definition}

Given $F$ a $(b,\rho)$-robust summand, and weights $(w_{ij})_{i,j\in [n]}$, one step of $F$-RG updates the client's parameter $(\vx_i)_{i \in [n]}$ according to
\[
\forall i \in \cH,\quad  \vy_i = \vx_i - \eta F( (\vx_i - \vx_j, w_{ij})_{j \in [n]}).
\]

One particular instance of $F$-RG simulated by a central server in a server–client architecture is the Nearest Neighbor Mixing (NNM) pre-aggregation scheme \citep{allouah2023fixing}. This method corresponds alternatively to NNA \citep{farhadkhani2023robust} or to $F$-RG with weights $w_{ij} = 1/(n-f)$ and step size $\eta = 1$ when $F$ is chosen as the geometrically trimmed sum (GTS), formally:
\[
\mathrm{GTS}\left((\vz_1,\nicefrac{1}{n-f}),\ldots, (\vz_n, \nicefrac{1}{n-f}) \right) := \frac{1}{n-f}\sum_{n-f \text{ smallest } \|\vz_i\|}\vz_i.
\]

\begin{proposition}
\label{app:prop:pre_processing_robust_avg}
    Let $\cG$ be the fully-connected graph on $[n]$ with uniform weights on the edges $\frac{1}{n-f}$. Let $F_1$ be a $(f,\nu)$-robust aggregator, and $F_2$ be a $(\rho,b)$ robust summand, with $b=\frac{f}{n-f}$. Then $F_1 \circ(F_2\text{-}RG)$ as in \cref{app:alg:F_RG_preaggregation} is a $(\tilde{\nu},f)$-robust aggregator, with 
    \[
    \tilde{\nu} = \delta (1+\nu)\quad \text{where} \quad \delta := \frac{2\rho f}{n-f}.
    \]
    We refer to $\delta$ as the mixing contraction factor.
\end{proposition}

\begin{proof}
    
    \citet{gaucher2025unified} show that, for graph $\cG$ of spectral gap $\gamma$, $\rm F_2-RG$, with $F_2$ a $(b,\rho)$-robust aggregator with step size $\eta = \mu_{\max}^{-1}(\cG_{\cH})$, ensures that, for $(\vy_i)_{i \in [n]} = \mathrm{F_2-RG}(\vx_i, i\in [n])$, and $b < \mu_2(\cG_{\cH})$
    \begin{equation}
    \label{app:frg_contraction}
    \frac{1}{|\cH|}\sum_{i\in \cH} \|\vy_i - \overline{\vx}_{\cH}\|^2 \le (1 - \gamma(1 - \delta))\frac{1}{|\cH|}\sum_{i \in \cH}\|\vx_i - \overline{\vx}_{\cH}\|^2,
    \end{equation}
    where $\delta = \frac{2\rho b}{\mu_2(\cG_{\cH})}$.
    
    Thus, plugging the $\vy_i$ in $F_2$ leads to 
    \begin{align*}
        \|F_2(\vy_1,\ldots,\vy_n) - \overline{\vx}_{\cH}\|^2 &\le (1 + \eps^{-1}) \|F_2(\vy_1,\ldots,\vy_n) - \overline{\vy}_{\cH}\|^2 + (1 + \eps) \|\overline{\vy}_{\cH} - \overline{\vx}_{\cH}\|^2 \\
        & \intertext[-4ex]{\hfill Young's inequality}\\
        &\le \nu (1 + \eps^{-1}) \frac{1}{|\cH|}\sum_{i\in\cH}\|\vy_i - \overline{\vy}_{\cH}\|^2 + (1 + \eps) \|\overline{\vy}_{\cH} - \overline{\vx}_{\cH}\|^2 \\
        & \intertext[-4ex]{\hfill $(f,\nu)$-robustness}\\
        &= (1+\nu)\left(\frac{1}{|\cH|}\sum_{i\in\cH}\|\vy_i - \overline{\vy}_{\cH}\|^2 + \|\overline{\vy}_{\cH} - \overline{\vx}_{\cH}\|^2 \right) \quad \text{for } \eps = \nu\\
        &\le (1+\nu)(1 - \gamma(1-\delta)) \frac{1}{|\cH|}\sum_{i \in \cH}\|\vx_i - \overline{\vx}_{\cH}\|^2\ , 
    \end{align*}   
    where the last inequality stems from \cref{app:frg_contraction} and a bias-variance decomposition.

    Thus, if the central server emulates one step of $F_2$-RG on a fully connected graph, with uniform weights $\frac{1}{|\cH|}$, we have that $\gamma=1$, $\mu_{\max}(\cG_{\cH}) = \mu_{\min}(\cG_{\cH}) = \eta^{-1} = 1$, and thus $\delta = \frac{2\rho f}{n-f}$, and use $F$ on the outputs (which corresponds to NNA for G$=$Trimmed Sum), we have that
    \[
    \tilde{\nu} = \delta(1+\nu).
    \]
\end{proof}

\section{From Hessian Similarity to Bounded Heterogeneity}

\new{In this section, we first show in \cref{app:proof:similarity_to_GB_heter} that the Hessian similarity assumption \cref{asmpt:Hessian_Similarity} implies the $(G,B)$-Heterogeneity assumption \cref{asmpt:GB_heterogeneity}, as stated in \cref{prop:similarity_to_GB_heter}. We then show in \cref{app:sec:dimension} that this technique can also be used to control the error with a tighter dependency on the dimension.}

\subsection{Proof of \cref{prop:similarity_to_GB_heter}}

\label{app:proof:similarity_to_GB_heter}
\begin{proof}
    
Let us write the differences of gradients as an integral.
\begin{align*}
    \nabla \cL_i(\vx) - \nabla \cL_{\cH}(\vx) &= \nabla \cL_i(\vx^*) + \int_{0}^1 \nabla^2\left(\cL_i - \cL_{\cH}\right)(\vx^* + t(\vx-\vx^*) ) \cdot (\vx-\vx^*)dt.
\end{align*}
 We now use Young's inequality, then apply Jensen's inequality. This allows us to leverage the similarity \cref{asmpt:pairwise_hessian_similarity}, which, along with the convexity of both $\cL_i$ and $\cL_{\cH}$, concludes the proof.
\begin{align*}
    \left\|\nabla \cL_i(\vx) - \nabla \cL_{\cH}(\vx)\right\|^2 \le \; & 2\left\|\nabla \cL_i(\vx^*)\right\|^2 \\
    & + 2\left\|\int_{0}^1 \nabla^2\left(\cL_i - \cL_{\cH}\right)(\vx^* + t(\vx-\vx^*) ) \cdot (\vx-\vx^*)dt\right\|^2\\
    \le \; & 2\left\|\nabla \cL_i(\vx^*)\right\|^2 \\
    & + 2\int_{0}^1 (\vx-\vx^*)^T (\nabla^2\left(\cL_i - \cL_{\cH}\right)(\vx^* + t(\vx-\vx^*) ))^2 \cdot (\vx-\vx^*)dt\\
    \le \; & 2\left\|\nabla \cL_i(\vx^*)\right\|^2 + 2 \Delta^2 \left\|\vx-\vx^*\right\|^2 
\end{align*}

Averaging it on all $i \in \cH$, and using that $\nabla \cL_{\cH}(\vx^*) \triangleq 0$, yields
\[
\frac{1}{|\cH|}\sum_{i \in \cH}\left\|\nabla \cL_i(\vx) - \nabla \cL_{\cH}(\vx)\right\|^2 \le \frac{2}{|\cH|}\sum_{i \in \cH}\|\nabla \cL_i(\vx^*)\|^2 + 2 \Delta^2 \left\|\vx-\vx^*\right\|^2 . 
\]
Using \cref{app:eq:cocoaxivity_strong_convexity}, written here as $ \left\|\vx-\vx^*\right\|^2 \le  \frac{1}{\mu^2}\|\nabla \cL(\vx)\|^2$, gives the result. 
\end{proof}

\subsection{Extension to a Dimensionally-Tight Definition of Robustness}

\label{app:sec:dimension}

If the notion $(f,\nu)$-robust aggregation rule is simple, and enjoys practical simple algorithms that achieve the criterion, it suffers from an obfuscation of the dependency on the dimension. Indeed, \cref{thm:allouha_lower_bound} is established using univariate counterexamples. As a consequence, it does not allows to understand the tightness of Byzantine-resilient algorithms for which the term $G^2=\frac{1}{|\cH|} \sum_{i\in\cH}\|\nabla \cL_i(\vx^*)\|^2$ might increase with the dimension. 

And indeed, a line of work \citep{zhu2023byzantine, allouah2023privacy, allouah2025towards}, inspired by the robust statistics literature, has shown that this dependency on the dimension can be alleviated, by suffering, instead of $G$, only a term $\tilde{G}:= \|\frac{1}{\cH} \nabla \cL_i(\vx^*) \nabla \cL_i(\vx^*)^T\|_{op}$, i.e. the largest eigenvalue of the empirical covariance matrix at the optimum.

This can be achieved by considering the broader notion of $(f,\nu,\|\cdot\|)$-robust aggregation rule. 
\begin{definition}
    \label{app:def:f_nu_norm_robust_aggregation}. Let $\|\cdot\|$ be a norm on symmetric positive semi-definite matrices. An aggregation rule $F=\R^{n\times d}\rightarrow \R^d$ is said to be $(f,\nu,\|\cdot\|)$ robust if, $\forall \, \vx_1,\ldots,\vx_n\in\R^d$, and any set $S\subset[n]$ of size $n-f$, 
    \[
    \|F(\vx_1,\ldots,\vx_n) - \overline{\vx}_S\|^2 \le \nu \cdot \| \Sigma((\vx_i)_{i\in S})\|, \]
    with $\Sigma((\vx_i)_{i\in S}) := \frac{1}{|S|}\sum_{i \in S} (\vx_i -\overline{\vx}_S)(\vx_i - \overline{\vx}_S)^T$.
\end{definition}

\begin{remark}
    As shown in \citet{allouah2025towards}, an adaptation of the celebrated Filter \citep{diakonikolas2017being} algorithm, coined Covariance-Agnostic Filter (CAF), is $(f,\nu = O(\nicefrac{f}{n-2f}),\|\cdot\|_{op})$ robust, where $\|\cdot\|_{op}$ is the operator norm for matrices, i.e., for PSD matrices the largest absolute eigenvalue.
\end{remark}
\begin{remark}
  Considering $(f,\nu, \|\cdot\|_{op})$-robust aggregation rule may lead to a $\sqrt{d}$ dimension improvement w.r.t. $(f,\nu)$-robust aggregation rule. Indeed, on the one hand we have that
    \[
    \|\Sigma((\vx_i)_{i\in S})\|_{F}\le \frac{1}{|S|}\sum_{i\in S}\|\vx_i - \overline{\vx}_S\|^2,
    \]
    and, on the other hand,
    \[
    \|\Sigma((\vx_i)_{i\in S})\|_{op}\le \|\Sigma((\vx_i)_{i\in S})\|_F \le \sqrt{d}\|\Sigma((\vx_i)_{i\in S})\|_{op},
    \]
    where both inequalities are tight for instances of PSD matrices. It follows that considering the robustness using the $(f,\nu,\|\cdot\|_{op})$ definition can give up to a $\sqrt{d}$ improvement when the covariance matrix is nearly isotropic. 
\end{remark}

\textbf{Inexact gradients can be tight with respect to the dimension.} Hopefully, the Inexact Oracle Framework works with $(f,\nu,\|\cdot\|_{op})$ robustness as well, and, in fact whenever $\|\cdot\|$ is unitary-invariant.

\begin{proposition}
\label{app:prop:similarity_to_GB_heter_dimensionnaly_tight}
    Let $\|\cdot\|$ be a unitary-invariant norm for PSD matrices. Under \cref{asmpt:pairwise_hessian_similarity}, if $(\cL_i)_{i\in\cH}$ are convex and $\cL_{\cH}$ is $\mu$-strongly convex, \cref{alg:reduction}\ with a $(f,\nu, \|\cdot\|)$-robust averaging rule (\cref{app:def:f_nu_norm_robust_aggregation}) yields:
    \begin{equation}
    \label{app:eq:inexact_gradient_inner_product}
    \|\tilde \nabla \cL_{\cH}(\vx) - \nabla \cL_{\cH}(\vx)\|^2 \le \tilde{G}^2 + \alpha \| \nabla \cL_{\cH}(\vx)\|^2 ,
    \end{equation}
    where $\alpha = \nu\frac{2\Delta^2}{\mu^2}$ and $\tilde{G}^2 = 2\nu \left\|\frac{1}{|\cH|}\sum_{i \in \cH}\nabla \cL_i(\vx^*)\nabla \cL_i(\vx^*)^T\right\|$.
\end{proposition}

\begin{corollary}
    Under \cref{asmpt:pairwise_hessian_similarity}, \cref{alg:reduction} with CAF \citep{allouah2025towards} yields a $(\zeta^2,\alpha)$-inexact oracle with
    \[
    \zeta^2 = O\paren{\frac{f}{n-2f} \left\|\frac{1}{\cH}\sum_{i\in\cH}\nabla\cL_i(\vx^*)\nabla\cL_i(\vx^*)^T\right\|_{op}} \quad \text{and}\quad \alpha = O\paren{\frac{f}{n-2f} \frac{\Delta^2}{\mu^2}}.
    \]
\end{corollary}

This in turn can be used in PIGS \cref{alg:FedProxyProx}, GD, or the Extra-gradient method \cref{alg:acc_inexact_extragradient} to achieve dimension-independent asymptotic errors.

\begin{proof}[Proof of \cref{app:prop:similarity_to_GB_heter_dimensionnaly_tight}]

To show the results, it is sufficient to show that 
\[
\|\Sigma((\nabla \cL(\vx))_{i\in\cH})\| \le 2 \|\Sigma((\nabla \cL_i(\vx^*))_{i\in \cH})\| + 2\frac{\Delta^2}{\mu^2} \|\nabla \cL_{\cH}(\vx)\|^2.
\]
We have that 
\begin{align*}
    \frac{1}{\cH}\sum_{i \in \cH}\big(\nabla \cL_i(\vx)& - \nabla \cL_{\cH}(\vx)\big)\left(\nabla \cL_i(\vx) - \nabla \cL_{\cH}(\vx)\right)^T 
    \\\preceq &  \frac{2}{\cH}\sum_{i \in \cH}\left(\nabla \cL_i(\vx) - \nabla \cL_{i}(\vx^*) - \nabla \cL_{\cH}(\vx) \right)\left(\nabla \cL_i(\vx) - \nabla \cL_{i}(\vx^*) - \nabla \cL_{\cH}(\vx)\right)^T
    \\& +\frac{2}{\cH}\sum_{i \in \cH}\nabla \cL_{i}(\vx^*)\nabla \cL_{i}(\vx^*)^T . 
\end{align*}

Now, using that for unitary-invariant norm for PSD matrices, if \(M \preceq N\), then \(\|M\|\le \|N\|\). Thus 

\begin{align*}
    \Bigg\|\frac{1}{\cH}\sum_{i \in \cH}\big(&\nabla \cL_i(\vx) - \nabla \cL_{\cH}(\vx)\big)\left(\nabla \cL_i(\vx) - \nabla \cL_{\cH}(\vx)\right)^T \Bigg\|
    \\\le &  \Bigg\|\frac{2}{\cH}\sum_{i \in \cH}\left(\nabla \cL_i(\vx) - \nabla \cL_{i}(\vx^*) - \nabla \cL_{\cH}(\vx) \right)\left(\nabla \cL_i(\vx) - \nabla \cL_{i}(\vx^*) - \nabla \cL_{\cH}(\vx)\right)^T\Bigg\|
    \\& +\Bigg\|\frac{2}{\cH}\sum_{i \in \cH}\nabla \cL_{i}(\vx^*)\nabla \cL_{i}(\vx^*)^T\Bigg\| .
\end{align*}

We control the first of the two terms as follows

\begin{align*}
   & \Bigg\|\frac{2}{\cH}\sum_{i \in \cH}\Big(\nabla \cL_i(\vx) - \nabla \cL_{i}(\vx^*) - \nabla \cL_{\cH}(\vx) \Big)\Big(\nabla \cL_i(\vx) - \nabla \cL_{i}(\vx^*) - \nabla \cL_{\cH}(\vx)\Big)^T\Bigg\| 
    \\\le& 
    \frac{2}{\cH}\sum_{i \in \cH} \Bigg\|\left(\nabla \cL_i(\vx) - \nabla \cL_{i}(\vx^*) - \nabla \cL_{\cH}(\vx) \right)\left(\nabla \cL_i(\vx) - \nabla \cL_{i}(\vx^*) - \nabla \cL_{\cH}(\vx)\right)^T\Bigg\|
    \\=& 
    \frac{2}{\cH}\sum_{i \in \cH} \left\|\int_{0}^1 \nabla^2\left(\cL_i - \cL_{\cH}\right)(\vx^* + t(\vx-\vx^*) ) \cdot (\vx-\vx^*)dt\right\|_2^2 \\
    & \intertext[-4ex]{\hfill  Using $\| \1_d \1_d^T\|= \1_d^T\1_d = d$}\\
    \\\le & 
    \frac{2}{\cH}\sum_{i \in \cH}  \int_{0}^1 (\vx-\vx^*)^T (\nabla^2\left(\cL_i - \cL_{\cH}\right)(\vx^* + t(\vx-\vx^*) ))^2 \cdot (\vx-\vx^*)dt \\
    & \intertext[-4ex]{\hfill  Using Jensen's inequality}
\end{align*}
Using \cref{asmpt:pairwise_hessian_similarity}, we obtain

\begin{align*}
    \Bigg\|\frac{2}{\cH}\sum_{i \in \cH}\Big(\nabla \cL_i(\vx) &- \nabla \cL_{i}(\vx^*) - \nabla \cL_{\cH}(\vx) \Big)\Big(\nabla \cL_i(\vx) - \nabla \cL_{i}(\vx^*) - \nabla \cL_{\cH}(\vx)\Big)^T\Bigg\| \\
     &\le
    \frac{2\Delta^2}{\cH}\sum_{i \in \cH}  \left\|\vx-\vx^*\right\|^2
     = 2\Delta^2  \left\|\vx-\vx^*\right\|^2.
\end{align*}

Eventually, combining the two terms yields
\[
\|\Sigma_{\cH}(\nabla \cL_i(\vx))\| \le 2\|\Sigma((\nabla \cL_i(\vx^*)_{i \in \cH})\| + 2\Delta^2 \left\|\vx-\vx^*\right\|^2. 
\]
Using \cref{app:eq:cocoaxivity_strong_convexity}, written here as $\left\|\vx-\vx^*\right\|^2 \le  \frac{1}{\mu^2}\|\nabla \cL(\vx)\|^2$, and using \cref{app:def:f_nu_norm_robust_aggregation} of $(f,\nu, \|\cdot\|)$ aggregators gives the result. 

\end{proof}

\section{Convergence of \cref{alg:FedProxyProx} and \cref{alg:acc_inexact_extragradient}}
\label{app:proofs}

In this section, we demonstrate the convergence of \cref{alg:FedProxyProx} and \cref{alg:acc_inexact_extragradient}. For the sake of generality, we consider the following problem:
\[
\min_{\vx\in\R^d}\{h(\vx) = p(\vx) + q(\vx)\}.
\]

\new{The guarantees of the Byzantine case then derives from considering $h = \cL_{\cH}$, $p = \cL_{\cH}-\hat{\cL}$ and $q = \hat{\cL}$.}

\subsection{Introduction of Standard Optimization tools}
\label{app:proofs_standard_tools}
In the proof, we use several standard optimization tools that we recall hereafter.
\begin{definition}[Bregman Divergence]
    The Bregman divergence of a real-valued function $h$ is defined as 
    \[
    D_h(\vx,\vy) := h(\vx) - h(\vy) - \langle \nabla h(\vy), \vx - \vy \rangle.
    \]
\end{definition}
The Bregman divergence has the following properties:
\begin{proposition}
    \begin{enumerate}
        \item A function $h$ is $\mu$-strongly convex and $L$-smooth, if and only if
        \begin{equation}
        \label{app:eq:smoothness_strongcvx_bregman}
        \forall \vx,\vy \in \R^d, \qquad \frac{\mu}{2}\|\vx - \vy\|^2 \le D_h(\vx;\vy) \le \frac{L}{2}\|\vx-\vy\|^2.
        \end{equation}
        \item \textbf{Three points identity.} Let $\vx,\vy,\vz \in \R^d$, it holds:
        \begin{equation}
        \label{app:eq:three_points_identity}
        D_{h}(\vx;\vy) + D_{h}(\vy;\vz) - D_{h}(\vx;\vz) = \langle  \nabla h(\vy) - \nabla h(\vz), \vy - \vx\rangle.
        \end{equation}
    \end{enumerate}
\end{proposition}

Furthermore, we will use the following property of strongly convex functions:

\begin{proposition}
    Let $h:\R^d \rightarrow \R$ be a $\mu$-strongly convex $L$-smooth loss function, then:
    \begin{equation}
        \label{app:eq:cocoaxivity_strong_convexity}
            \langle\nabla h(\vx) - \nabla h(\vy), \vx - \vy\rangle \le \frac{1}{\mu}\|\nabla h(\vx) - \nabla h(\vy)\|^2.
        \end{equation}
\end{proposition}
\begin{proof}
        We use that the Fenchel transform of $h$, defined as $h^*:\vu \rightarrow \sup_{\vx \in \R^d}\{\langle \vx,\vu\rangle - h(\vx)\}$, is $1/\mu$ smooth. 
        This ensures that 

        Then, using Fenchel's identity, which here can be written as $\nabla h = (\nabla h^*)^{-1}$, and considering $\vx = \nabla h^*(\vu)$ and $\vy = \nabla h^*(\vv)$ yields the result.
\end{proof}

\subsection{Convergence of PIGS (\cref{alg:FedProxyProx})}
\label{app:proof:proxyprox}
We show the following theorem, of which \cref{thm:proxyprox_conv} is a direct consequence.

\begin{theorem}
\label{app:thm:proxy_prox_inexact}
   Let $h:\R^d \rightarrow \R$ be a $\mu$-strongly convex and $L_h$-smooth loss function, written as $h = p + q$, where $p$ is $L_p$-smooth (i.e. $\|\nabla^2 p(\vx)\|_{op} \le L_p$). 
Assume that we have access to an inexact gradient oracle $\tilde \nabla p$ satisfying
\begin{equation}
\label{app:eq:inexact_gradient_scalar_product}
\forall \vx\in\R^d, \quad\|\tilde \nabla p(\vx) - \nabla p(\vx)\|^2 \le \zeta^2 + \beta \mu^2\sqn{\vx - \vx^*},
\end{equation}
where \(\vx^* := \argmin_{\vx\in \R^d}h(\vx)\) and $\beta < \nicefrac{1}{8}$.
   Consider $(\vx_{k})_{k\ge 0}$ such that for any $k$, given 
   \begin{equation}
   \label{app:eq:def_phi_k}
       \phi_{\eta}^k(\vx) := \langle \tilde{\nabla} p(\vx_k) + \nabla q(\vx_k), \vx \rangle + D_{q}(\vx;\vx_k) + \frac{1}{2\eta}\|\vx - \vx_k\|^2,
   \end{equation}
$\vx_{k+1}$ is chosen such that
\begin{equation}
\label{app:eq:error_in_prox}
\|\nabla \phi_{\eta}^k (\vx_{k+1})\|^2 \le c\|\vx_{k+1} - \vx_k\|^2 + E^2.
\end{equation}

Denote $s_k = (1 + \frac{\eta \mu}{3})^{k}$, and $S_k = \sum_{k=0}^K s_k$. Then for $\eta \le \frac{1}{L_p + 4c/\mu}$, and $\hat{\vx}_k := \frac{1}{S_k}\sum_{k=0}^K s_k \vx_k$, we have 

\[h(\hat\vx_{K}) - h(\vx^*) \le \left(1 + \frac{\eta \mu}{3}\right)^{1-K}\frac{2}{\eta(1-4\beta)}\|\vx_0 - \vx^*\|^2 + 2\frac{\zeta^2 + E^2}{\mu(1-4\beta)}.\]
\end{theorem}

Note that we did lose some constant factors for the sake of clarity.

\subsubsection{Proof of \cref{app:thm:proxy_prox_inexact}}
This proof is an adaptation of the proof of \citet{woodworth2023two}, in which we modify the Lyapunov function to allow for the multiplicative error term (i.e $\beta \neq 0$) in the inexact gradient assumption. 

We first show the following decomposition

\begin{lemma}
\label{app:lemma:proof_proxyprox_1}
Denote $\ve_k := \nabla p(\vx_k) - \tilde \nabla p(\vx_k)$, such that, by definition $\|\ve_k\|^2 \le \zeta^2 + \beta \mu \langle \nabla h(\vx_k),\vx_k - \vx^*\rangle$. Then the following identity holds:
    \begin{align*}
        h(\vx_{k+1}) - h(\vx^*) = \; &  \left(\frac{1}{2\eta}\|\vx_{k}- \vx^*\|^2 - D_{p}(\vx^*;\vx_{k})\right) \\
        &- \left(\frac{1}{2\eta}\|\vx_{k+1}- \vx^*\|^2 - D_{p}(\vx^*;\vx_{k+1}) \right) \\
        &- D_h(\vx^*;\vx_{k+1})\\
        & + \langle \nabla \phi_{\eta}^k(\vx_{k+1}) + \ve_k, \vx_{k+1} - \vx^* \rangle \\
        &+ D_p(\vx_{k+1};\vx_k) - \frac{1}{2\eta}\|\vx_{k}- \vx_{k+1}\|^2.
    \end{align*}
\end{lemma}
The first and second line are telescopic; the third line will be leveraged to ensure a strict decrease in the Lyapunov value; the fourth line corresponds to the error due to the approximate proximal step and the erroneous gradients. Eventually, the last line corresponds to a discretization error.

\begin{proof}

    Computing $\nabla \phi_{\eta}^k(\vx_{k+1})$ using \cref{app:eq:def_phi_k} and $\nabla_\vx D_q(\vx,\vx_k) = \nabla q(\vx) - \nabla q(\vx_k)$, defining $\ve_k = \nabla p(\vx_{k}) - \tilde \nabla p(\vx_k)$ the error in the computation of $\nabla p(\vx_k)$, and rearranging terms yields
    \begin{align*}
        \nabla \phi_{\eta}^k(\vx_{k+1}) = \; & \tilde \nabla p(\vx_{k}) + \nabla q(\vx_{k}) + \nabla q(\vx_{k+1}) - \nabla q(\vx_k) + \frac{1}{\eta}(\vx_{k+1} - \vx_k)\\
        = \; & \underbrace{\tilde{\nabla} p(\vx_k) - \nabla p(\vx_k)}_{= - \ve_k} + \nabla p(\vx_k) - \nabla p(\vx_{k+1}) \\
        & + \underbrace{\nabla p(\vx_{k+1}) + \nabla q(\vx_{k+1}) }_{= \nabla h(\vx_{k+1})} + \frac{1}{\eta}(\vx_{k+1}-\vx_k)
        \\
        \nabla h(\vx_{k+1}) = \; & \nabla p(\vx_{k+1}) - \nabla p(\vx_k)  + \nabla \phi_{\eta}^k(\vx_{k+1}) + \ve_k + \frac{1}{\eta}(\vx_{k}-\vx_{k+1})\ .
    \end{align*}

    Furthermore, the definition of the Bregman divergence implies 
    \[
    h(\vx_{k+1}) - h(\vx^*) = \langle \nabla h(\vx_{k+1}), \vx_{k+1}- \vx^*\rangle - D_h(\vx^*; \vx_{k+1}).
    \]
    Substituting the previous identity into the latter one yields
    \begin{align*}
        h(\vx_{k+1}) \; - \; & h(\vx^*)\\
        = \; & \langle \nabla h(\vx_{k+1}), \vx_{k+1}- \vx^*\rangle - D_h(\vx^*; \vx_{k+1})\\
        = \; & \left \langle \nabla p(\vx_{k+1}) - \nabla p(\vx_k)  + \nabla \phi_{\eta}^k(\vx_{k+1}) + \ve_k + \frac{1}{\eta}(\vx_{k}-\vx_{k+1})
        , \vx_{k+1}- \vx^*\right \rangle \\
        & - D_h(\vx^*; \vx_{k+1})\\
        = \; &  D_{p}(\vx^*;\vx_{k+1}) + D_p(\vx_{k+1};\vx_k) - D_{p}(\vx^*;\vx_{k}) \\
        & + \langle \nabla \phi_{\eta}^k(\vx_{k+1}) + \ve_k, \vx_{k+1} - \vx^* \rangle \\
        & - D_h(\vx^*;\vx_{k+1}) + \frac{1}{2\eta}\|\vx_{k}- \vx^*\|^2 - \frac{1}{2\eta}\|\vx_{k}- \vx_{k+1}\|^2 - \frac{1}{2\eta}\|\vx_{k+1}- \vx^*\|^2.
    \end{align*}

    Where we used the parallelogram identity,
    \[
    \frac{1}{\eta}\langle \vx_k - \vx_{k+1}, \vx_{k+1} - \vx^*\rangle = \frac{1}{2\eta}\|\vx_k  - \vx^*\|^2 - \frac{1}{2\eta}\|\vx_k - \vx_{k+1}\|^2 - \frac{1}{2\eta}\|\vx_{k+1} - \vx^*\|^2,
    \]
    and the three-point identity \cref{app:eq:three_points_identity},
    \[
    D_{p}(\vx^*;\vx_{k+1}) + D_p(\vx_{k+1};\vx_k) - D_{p}(\vx^*;\vx_{k}) = \langle  \nabla p(\vx_{k+1}) - \nabla p(\vx_k), \vx_{k+1} - \vx^*\rangle.
    \]
\end{proof}

We now control the error due to the inexact oracle and the approximate proximal step as follows:
\begin{lemma}
    \label{app:lemma:proof_proxyprox_2}
    For any $\eps >1$, the following inequality holds, for any $k\ge 0$,
    \begin{align*}
    & \langle \nabla \phi_{\eta}^k(\vx_{k+1}) + \ve_k, \vx_{k+1} - \vx^*\rangle \\
    \le \; & \frac{c}{2\eps}\|\vx_{k+1} - \vx_k\|^2  + \frac{\beta \mu^2 }{2\eps} \| \vx_k - \vx^* \|^2 + \eps\|\vx_{k+1}-\vx^*\|^2 + \frac{E^2 + \zeta^2}{2\eps}.
    \end{align*}
\end{lemma}

\begin{proof}
    Writing the inexact gradient assumption (\cref{app:eq:inexact_gradient_scalar_product}) and Young's inequality, we have for any $\eps >0$,
    \begin{equation}
    \label{app:eq:control_Young_inexact_gd}
        \langle \ve_k, \vx_{k+1} - \vx^*\rangle \le \frac{1}{2\eps}\left(\zeta^2 + \beta \mu^2  \sqn{\vx_k - \vx^*}\right) + \frac{\eps}{2}\|\vx_{k+1}-\vx^*\|^2.
    \end{equation}
    Using the assumption \cref{app:eq:error_in_prox} on the proximal step, namely
    \[
    \|\nabla \phi_{\eta}^k(\vx_{k+1}\|^2 \le c \|\vx_{k+1} - \vx_k\|^2 + E^2,
    \]
    and Young's inequality, we have that 
    \begin{equation}
    \label{app:eq:control_Young_prox}
        \langle \nabla \phi_{\eta}^k(\vx_{k+1}), \vx_{k+1} - \vx^*\rangle \le \frac{1}{2\eps}\left(c\|\vx_{k+1} - \vx_k\|^2 + E^2\right) + \frac{\eps}{2}\|\vx_{k+1}-\vx^*\|^2.
    \end{equation}
    Putting \cref{app:eq:control_Young_inexact_gd} and \cref{app:eq:control_Young_prox} together gives \cref{app:lemma:proof_proxyprox_2}.
\end{proof}

We now define a Lyapunov energy function and show its decrease.

\begin{lemma}
\label{app:lemma:proof_proxyprox_3}
    Let's consider the Lyapunov 
    \begin{equation}
    \label{app:eq:def_Lyapunov}
    \Gamma_{k} := \Big(\frac{1}{2\eta}+ \frac{\beta \mu^2}{2 \eps}\Big)\|\vx_k - \vx^*\|^2 - D_{p}(\vx^*;\vx_k).
    \end{equation}
    Then, under the choice $\eps = \frac{\mu}{4}$ and $\eta \le 1/(L_p+\nicefrac{4c}{\mu})$, we have
    \begin{align*}
        \left(1 - 4\beta\right) &\left(h(\vx_{k+1}) - h(\vx^*)\right) 
    \le \Gamma_k - \left( 1 + \frac{\eta \mu}{3}\right)\Gamma_{k+1}
    +2\frac{E^2 + \zeta^2}{\mu}. 
    \end{align*}
\end{lemma}
\begin{proof}
    Combining \cref{app:lemma:proof_proxyprox_1,app:lemma:proof_proxyprox_2}, gives 
    \begin{align*}
        h(\vx_{k+1}) - h(\vx^*) \le \; & \left(\frac{1}{2\eta}\|\vx_{k}- \vx^*\|^2 - D_{p}(\vx^*;\vx_{k})\right) \\
        & - \left(\frac{1}{2\eta}\|\vx_{k+1}- \vx^*\|^2 - D_{p}(\vx^*;\vx_{k+1}) \right) \\
        & - D_h(\vx^*;\vx_{k+1})\\
        & + \frac{c}{2\eps}\|\vx_{k+1} - \vx_k\|^2  + \frac{\beta \mu^2 }{2\eps} \sqn{\vx_k - \vx^* } + \eps\|\vx_{k+1}-\vx^*\|^2 + \frac{E^2 + \zeta^2}{2\eps}\\
        & + D_p(\vx_{k+1};\vx_k) - \frac{1}{2\eta}\|\vx_{k}- \vx_{k+1}\|^2.
    \end{align*}
    
    Thus, removing on both sides $\frac{\beta \mu^2}{2\eps}\sqn{\vx_{k+1} -\vx^*}$ and rearranging yields
    \begin{align*}
         \left(h(\vx_{k+1}) - h(\vx^*)\right)- \frac{\beta \mu^2}{2\eps}&\sqn{\vx_{k+1} -\vx^*}
         \\
         \le&
         \left(\Big(\frac{1}{2\eta} + \frac{\beta \mu^2}{2\eps}\Big)\|\vx_{k}- \vx^*\|^2 - D_{p}(\vx^*;\vx_{k})\right) 
        \\&- \left(\Big(\frac{1}{2\eta} + \frac{\beta \mu^2}{2\eps}\Big)\|\vx_{k+1}- \vx^*\|^2 - D_{p}(\vx^*;\vx_{k+1}) \right)\\
         &- D_h(\vx^*;\vx_{k+1})\\
        & + \frac{c}{2\eps}\|\vx_{k+1} - \vx_k\|^2 + \eps\|\vx_{k+1}-\vx^*\|^2 +  \frac{E^2 + \zeta^2}{2\eps} \\
        & + D_{p}(\vx_{k+1};\vx_k ) - \frac{1}{2\eta}\|\vx_{k} - \vx_{k+1}\|^2 . 
    \end{align*}
    Which, using the definition of $\Gamma_k$ (\cref{app:eq:def_Lyapunov}), yields
    \begin{align}
         h(\vx_{k+1}) - h(\vx^*)- \frac{\beta \mu^2}{2\eps}\sqn{\vx_{k+1} -\vx^*} &\le \Gamma_k - \Gamma_{k+1} -  D_h(\vx^*;\vx_{k+1}) \notag\\
    & + \frac{c}{2\eps}\|\vx_{k+1} - \vx_k\|^2 + \eps\|\vx_{k+1}-\vx^*\|^2 \notag \\
    & +  \frac{E^2 + \zeta^2}{2\eps} \notag\\
    & + D_{p}(\vx_{k+1};\vx_k ) - \frac{1}{2\eta}\|\vx_{k} - \vx_{k+1}\|^2. \label{app:eq:decomposition_withlyap_1}
    \end{align}

    We now use the regularity property of $h$ and $p$ to derive the following inequalities: 
    \begin{align}
        D_{h}(\vx^*;\vx_{k+1}) &\ge \frac{\mu}{2}\|\vx_{k+1}-\vx^*\|^2, &&\text{since $h$ is $\mu$-strongly convex}\label{app:eq:usage_strong_cvx_h}
        \\
        h(\vx_{k+1}) - h(\vx^*) & \ge \frac{\mu}{2} \|\vx_{k+1} - \vx^*\|,&&\text{since $\nabla h(\vx^*) = 0$ } \label{app:eq:usage_strong_cvx_h_star}
        \\
        |D_{p}(\vx^*;\vx_{k+1})| &\le \frac{L_p}{2}\|\vx_{k+1}-\vx^*\|^2,&&\text{since $p$ is $L_p$-smooth. }\label{app:eq:usage_smooth_h}
    \end{align}
    And, using \cref{app:eq:usage_smooth_h}, we have, for $\eta^{-1} \ge L_p$,
    \[
    0 \le \frac{1}{2\eta}\|\vx_{k+1} - \vx^*\|^2 - D_p(\vx^*,\vx_{k+1}) \le \eta^{-1}\|\vx_{k+1} - \vx^*\|^2.
    \]
    Which yields, %
    \begin{align}
        0 \le \Gamma_{k+1} &\le \Big(\eta^{-1}+\frac{\beta \mu^2}{2 \eps}\Big)\|\vx_{k+1} - \vx^*\|^2 \label{app:eq:usage_regularity_on_Gamma}
    \end{align}
    
Combining \cref{app:eq:usage_strong_cvx_h} and \cref{app:eq:usage_regularity_on_Gamma} yields
\begin{equation}
\label{app:eq:lowerbound_Dh}
D_h(\vx^*;\vx_{k+1}) \ge  \frac{\mu}{4}\|\vx_{k+1}-\vx^*\|^2 +\frac{\mu}{2}\left(\eta^{-1} + \frac{\beta \mu^2}{2 \eps}\right)^{-1} \Gamma_{k+1}.
\end{equation}

Additionally, \cref{app:eq:usage_strong_cvx_h_star} implies that 
\begin{equation}
\label{app:eq:lowerbound_h_hstar}
    h(\vx_{k+1}) - h(\vx^*)- \frac{\beta \mu^2}{2\eps}\sqn{\vx_{k+1} -\vx^*} \ge \left(1 - \frac{\beta \mu}{\eps}\right)\left(h(\vx_{k+1}) - h(\vx^*)\right)
\end{equation}

We can now use \cref{app:eq:usage_smooth_h}, \cref{app:eq:lowerbound_Dh} and \cref{app:eq:lowerbound_h_hstar} into \cref{app:eq:decomposition_withlyap_1}.
\begin{align*}
    \left(1 - \frac{\beta \mu}{\eps}\right) &\left(h(\vx_{k+1}) - h(\vx^*)\right) \\ 
    &\le \Gamma_k - \Gamma_{k+1} \\
    &-\frac{\mu}{4}\|\vx_{k+1}-\vx^*\|^2 - \frac{\mu}{2}\left(\eta^{-1} + \frac{\beta \mu^2}{2 \eps}\right)^{-1} \Gamma_{k+1}\notag\\
    & + \frac{c}{2\eps}\|\vx_{k+1} - \vx_k\|^2 + \eps\|\vx_{k+1}-\vx^*\|^2 +  \frac{E^2 + \zeta^2}{2\eps} \notag\\
    & + D_{p}(\vx_{k+1};\vx_k ) - \frac{1}{2\eta}\|\vx_{k} - \vx_{k+1}\|^2. 
\end{align*}
Using the smoothness of $p$ and rearranging yields 

\begin{align*}
    \left(1 - \frac{\beta \mu}{\eps}\right) &\left(h(\vx_{k+1}) - h(\vx^*)\right) \\ 
    &\le \Gamma_k - \Gamma_{k+1} \\
    &\left( \eps-\frac{\mu}{4}\right)\|\vx_{k+1}-\vx^*\|^2 - \frac{\mu}{2}\left(\eta^{-1} + \frac{\beta \mu^2}{2 \eps}\right)^{-1} \Gamma_{k+1}\\
    & + \left( \frac{L_p}{2} + \frac{c}{2\eps} - \frac{1}{2\eta}\right)\|\vx_{k} - \vx_{k+1}\|^2\\
    &+\frac{E^2 + \zeta^2}{2\eps}. \\
\end{align*}

Choosing $\eps=\frac{\mu}{4}$ and enforcing $\eta \le 1/(L_p + \frac{4c}{\mu})$ yields 

\begin{align*}
    \left(1 - 4\beta\right) &\left(h(\vx_{k+1}) - h(\vx^*)\right) 
    \le \Gamma_k - \left( 1 + \frac{\eta \mu}{2+4\beta \eta \mu}\right)\Gamma_{k+1}
    +2\frac{E^2 + \zeta^2}{\mu}. 
\end{align*}
Which, considering that for $\beta < 1/4$, and $\mu < L_p$, we have 
\[
\frac{\eta\mu}{2 + 4\beta \eta \mu} > \frac{\eta \mu}{3},
\]
gives the desired result.
\end{proof}

We can now prove \cref{app:thm:proxy_prox_inexact}.

\begin{proof}
We denote $s_k := \left(1 + \frac{\eta \mu}{3}\right)^{k-1}$ and $S_k := \sum_{k=1}^K s_k$ and considering
\[
\hat{\vx}_{K} := \frac{1}{S_k}\sum_{k=1}^K s_k \vx_k,
\]
allows us to write, by convexity of $h$,
\begin{align*}
    (1-4\beta)h(\hat{\vx}_K) - h^* &\le \frac{1-4\beta}{S_k}\sum_{k=1}^K s_k (h(\vx_k) - h^*)\\
    &\le \left(\frac{\Gamma_{0}}{S_k} + \frac{2\zeta^2}{\mu} + \frac{2E^2}{\mu}\right) \\
    & \intertext[-4ex]{\hfill using \cref{app:lemma:proof_proxyprox_3}.}
\end{align*}

Now, using \cref{app:eq:usage_regularity_on_Gamma}, we have that 
\[
\Gamma_0 \le (\frac{1}{\eta} + 2\beta \mu)\|\vx_0- \vx^*\|^2 \le \frac{2}{\eta}\|\vx_0- \vx^*\|^2,
\]
and considering that
\begin{align*}
    S_k &= \frac{(1+\frac{\eta \mu}{3})^K - 1}{\eta \mu /3} = \frac{(1+\frac{\eta \mu}{3})^{K-1}\frac{\eta \mu}{3} + (1+\frac{\eta \mu}{3})^{K-1} - 1}{\eta \mu /3}\\
    & \le \left(1 + \frac{\eta \mu }{3}\right)^{K-1},
\end{align*}
we have 
\begin{align*}
h(\hat\vx_{K}) - h(\vx^*) &\le \left(1 + \frac{\eta \mu}{3}\right)^{1-K}\frac{\Gamma_0}{1-4\beta} + 2\frac{\zeta^2 + E^2}{\mu(1-4\beta)} \\
h(\hat\vx_{K}) - h(\vx^*) &\le \left(1 + \frac{\eta \mu}{3}\right)^{1-K}\frac{2}{\eta(1-4\beta)}\|\vx_0 - \vx^*\|^2 + 2\frac{\zeta^2 + E^2}{\mu(1-4\beta)}.
\end{align*}
Which concludes the proof.
\end{proof}

\subsection{Convergence of the Proximal Accelerated Inexact Extra-gradient Method (\cref{alg:acc_inexact_extragradient})}
\label{app:proof:extragradient}

\begin{algorithm}[h]
\caption{Accelerated Inexact Extra-gradient}
\label{app:alg:acc_inexact_extragradient}
	\begin{algorithmic}[1]
		\STATE {\bf Input:} $\vx^0=\vx_f^0 \in \R^d$
		\STATE {\bf Parameters:} $\tau \in (0,1)$, $\eta,\theta,\gamma > 0, K \in \{1,2,\ldots\}$
		\FOR{$k=0,1,2,\ldots, K-1$}
			\STATE $\vx_g^k = \tau \vx^k + (1-\tau) \vx_f^k$\label{app:ae:line:1}
			\STATE $\vx_f^{k+1} \approx \argmin_{\vx \in \R^d}\left[ \phi_\theta^k(\vx) := p(\vx_g^k) + \langle\tilde{\nabla} p(\vx_g^k),x - \vx_g^k\rangle + \frac{1}{2\theta}\sqn{\vx - \vx_g^k} + q(\vx)\right]$\label{app:ae:line:2}%
			\STATE $\vx^{k+1} = \vx^k + \eta\gamma (\vx_f^{k+1}  - \vx^k)- \eta \tilde{\nabla} h(\vx_f^{k+1})$\label{app:ae:line:3} %
		\ENDFOR
		\STATE {\bf Output:} $\vx^K$
	\end{algorithmic}
\end{algorithm}

\begin{lemma}
\label{app:lemma:proof_extragrad}
	 Assume $h$ is $\mu$-strongly convex, and $p$ and $ q$ are respectively $L_p$ and $L_q$  smooth. Consider  \cref{app:alg:acc_inexact_extragradient} with $\theta = \frac{1}{2\sqrt{2}L_p}$, the following inequality holds for all $\bar{\vx} \in \R^d$
	\begin{equation}
	\label{ae:eq:1}
	\begin{split}
		2\langle\bar{\vx} - \vx_g^k, \nabla h(\vx_f^{k+1})\rangle
		\leq&
		2\left[h(\bar{\vx}) - h(\vx_f^{k+1})\right] - \mu \sqn{\vx_f^{k+1} - \bar{\vx}}
		- \theta\sqn{\nabla h(\vx_f^{k+1})}
		\\&
		+2\theta\left(\sqn{\nabla \phi_\theta^k(\vx_f^{k+1})} - \frac{1}{4\theta^2}\sqn{\vx_f^{k+1} - \vx_g^k}\right)
        \\& 
        + 4\theta (\zeta^2 + \alpha \sqn{\nabla h(\vx_g^k)}).
	\end{split}
	\end{equation}
\end{lemma}

\begin{proof}
Using $\mu$-strong convexity of $h(\vx)$, we get:
	\begin{align*}
		2\langle\bar{\vx} - \vx_g^k, \nabla h(\vx_f^{k+1})\rangle
		= \; &
		2\langle\bar{\vx} - \vx_f^{k+1}, \nabla h(\vx_f^{k+1})\rangle
		+2\langle\vx_f^{k+1} - \vx_g^k, \nabla h(\vx_f^{k+1})\rangle
		\\ \leq \; &
		2\left[h(\bar{\vx}) - h(\vx_f^{k+1})\right] - \mu \sqn{\vx_f^{k+1} - \bar{\vx}} \\
		& + 2\langle\vx_f^{k+1} - \vx_g^k, \nabla h(\vx_f^{k+1})\rangle
		\\ 
        = \; &
		2\left[h(\bar{\vx}) - h(\vx_f^{k+1})\right] - \mu \sqn{\vx_f^{k+1} - \bar{\vx}} \\
		& + 2\theta\langle\theta^{-1}(\vx_f^{k+1} - \vx_g^k), \nabla h(\vx_f^{k+1})\rangle
		\\ = \; &
		2\left[h(\bar{\vx}) - h(\vx_f^{k+1})\right] - \mu \sqn{\vx_f^{k+1} - \bar{\vx}}
		\\&
		-\frac{1}{\theta}\sqn{\vx_f^{k+1} - \vx_g^k} - \theta\sqn{\nabla h(\vx_f^{k+1})}
		\\&
		+\theta\sqn{\theta^{-1}(\vx_f^{k+1} - \vx_g^k) + \nabla h(\vx_f^{k+1})}.
	\end{align*}
The definition of $\phi_\theta^k(\vx)$, and the $L_p$-Lipschitzness of $\nabla p$ give:
	\begin{align*}
		2\langle\bar{\vx} - \vx_g^k, \nabla h(\vx_f^{k+1})\rangle
		\leq \; &
		2\left[h(\bar{\vx}) - h(\vx_f^{k+1})\right] - \mu \sqn{\vx_f^{k+1} - \bar{\vx}} \\
        & - \frac{1}{\theta}\sqn{\vx_f^{k+1} - \vx_g^k} - \theta\sqn{\nabla h(\vx_f^{k+1})}
		\\ &
		+\theta\sqn{\nabla \phi_\theta^k(\vx_f^{k+1}) + \nabla p(\vx_f^{k+1}) - \nabla p (\vx_g^k) + \nabla p (\vx_g^k) -\tilde{\nabla} p (\vx_g^k) }
		\\ \leq \; &
		2\left[h(\bar{\vx}) - h(\vx_f^{k+1})\right] - \mu \sqn{\vx_f^{k+1} - \bar{\vx}} \\
		& - \frac{1}{\theta}\sqn{\vx_f^{k+1} - \vx_g^k} - \theta\sqn{\nabla h(\vx_f^{k+1})}
		\\ &
		+2\theta\sqn{\nabla \phi_\theta^k(\vx_f^{k+1})}
		+4\theta L_p^2\sqn{\vx_f^{k+1} - \vx_g^k} \\
        & + 4\theta \|\tilde{\nabla} p (\vx_g^k) -\nabla p (\vx_g^k)\|^2
		\\ = \; &
		2\left[h(\bar{\vx}) - h(\vx_f^{k+1})\right] - \mu \sqn{\vx_f^{k+1} - \bar{\vx}} \\
		& - \frac{1}{\theta}\left(1 - 4\theta^2L_p^2\right) \sqn{\vx_f^{k+1} - \vx_g^k} - \theta\sqn{\nabla h(\vx_f^{k+1})}
		\\ &
		+2\theta\sqn{\nabla \phi_\theta^k(\vx_f^{k+1})} + 4\theta (\zeta^2 + \alpha\|\nabla h(\vx_g^k)\|^2)
	\end{align*}
With $\theta = \frac{1}{2\sqrt{2}L_p}$, we have
	\begin{align*}
		2\langle\bar{\vx} - \vx_g^k, \nabla h(\vx_f^{k+1})\rangle
		\leq&
		2\left[h(\bar{\vx}) - h(\vx_f^{k+1})\right] - \mu \sqn{\vx_f^{k+1} - \bar{\vx}}
		-\frac{1}{2\theta}\sqn{\vx_f^{k+1} - \vx_g^k}
		\\&
		- \theta\sqn{\nabla h(\vx_f^{k+1})}
		+2\theta\sqn{\nabla \phi_\theta^k(\vx_f^{k+1})}
        + 4\theta (\zeta^2 + \alpha \sqn{\nabla h(\vx_g^k)}).
	\end{align*}

This completes the proof of the lemma.
\end{proof}

\begin{theorem}[\cref{thm:inexact_extragradient_convergence}]
Assume that $h=p+q$ is $\mu$-strongly convex and $L_h$-smooth, that $p$ and $q$ are respectively $L_p$ and $L_q$-smooth, and that $\alpha < \min\left(\frac{L_p^2}{4L_h^2},\frac{1}{16}\right)$. Under the parametrization 
\[
\gamma = \mu,\quad \theta = \frac{1}{2\sqrt{2}L_p},\quad \eta = \min\left(\frac{1}{2\mu}, \frac{\theta}{24 \tau}\right), \quad \text{and} \quad \tau = \min\left(\frac{\theta \mu}{16 \alpha}, \sqrt{\frac{\theta \mu}{48}}\right),
\]
\cref{app:alg:acc_inexact_extragradient} with \(
\sqn{\nabla \phi_\theta^k(\vx_f^{k+1})} \le  \frac{1}{8\theta^2}\sqn{\vx_f^{k+1} - \vx_g^k},
\) yields 
\[
h(\vx_f^K) - h(\vx^*) \le O\Big((1 - \rho)^K C + \frac{\zeta^2}{\mu}\Big),
\]
where $\rho = O\left(\min\left(\frac{1}{\sqrt{\kappa}}, \frac{\mu}{\alpha L_p}\right)\right)$ and $R = \tau^2L_p\|\vx^0 - \vx^*\|^2 + h(\vx^0) - h(\vx^*)$.
\end{theorem}

\begin{proof}
Using line 6 of \cref{app:alg:acc_inexact_extragradient}, we get
\begin{align*}
		\frac{1}{\eta}\sqn{\vx^{k+1} - \vx^*}
		= \; &
		\frac{1}{\eta}\sqn{\vx^k - \vx^*}
		+\frac{2}{\eta} \langle \vx^{k+1} - \vx^k, \vx^k - \vx^*\rangle
		+\frac{1}{\eta}\sqn{\vx^{k+1} - \vx^k}
		\\ = \; &
		\frac{1}{\eta}\sqn{\vx^k - \vx^*}
		+2\gamma\langle\vx_f^{k+1} -\vx^k, \vx^k - \vx^*\rangle
		\\&
		-2\langle\nabla h(\vx_f^{k+1}), \vx^k - \vx^*\rangle
		+\frac{1}{\eta}\sqn{\vx^{k+1} - \vx^k} \\
        & + 2\langle\nabla h(\vx_f^{k+1}) - \tilde{\nabla} h(\vx_f^{k+1}), \vx^k - \vx^*\rangle
		\\ \le \; &
		\frac{1}{\eta}\sqn{\vx^k - \vx^*}
		+\gamma\sqn{\vx_f^{k+1} - \vx^*}
		-\gamma\sqn{\vx^k - \vx^*}
		-\gamma\sqn{\vx_f^{k+1} - \vx^k}
		\\&
		-2\langle\nabla h(\vx_f^{k+1}), \vx^k - \vx^*\rangle
		+\frac{1}{\eta}\sqn{\vx^{k+1} - \vx^k}.
        \\& + \varepsilon \|\nabla h(\vx_f^{k+1}) - \tilde{\nabla} h(\vx_f^{k+1})\|^2 + \varepsilon^{-1} \sqn{\vx^k - \vx^*}
\end{align*}
Line 4 of \cref{app:alg:acc_inexact_extragradient} gives
	\begin{align*}
		\frac{1}{\eta}\sqn{\vx^{k+1} - \vx^*}
		\le \; &
		\left(\frac{1}{\eta}-\gamma+\varepsilon^{-1}\right)\sqn{\vx^k - \vx^*}
		+\gamma\sqn{\vx_f^{k+1} - \vx^*}
		+\frac{1}{\eta}\sqn{\vx^{k+1} - \vx^k} \\
        & - \gamma\sqn{\vx_f^{k+1} - \vx^k}
        + 2\langle\nabla h(\vx_f^{k+1}), \vx^* - \vx_g^k\rangle \\
        & + \frac{2(1-\tau)}{\tau}\langle\nabla h(\vx_f^{k+1}), \vx_f^k - \vx_g^k\rangle
        + \varepsilon(\zeta^2 + \alpha \sqn{\nabla h(\vx_f^{k+1})}).
	\end{align*}
Using \cref{app:lemma:proof_extragrad} with $\bar{\vx} = \vx^*$ and $\bar{\vx} = \vx_f^k$, we get

\begin{align*}
		\frac{1}{\eta}\sqn{\vx^{k+1} - \vx^*}
		\leq \; &
		\left(\frac{1}{\eta}-\gamma + \varepsilon^{-1}\right)\sqn{\vx^k - \vx^*}
		+\gamma\sqn{\vx_f^{k+1} - \vx^*}
		+\frac{1}{\eta}\sqn{\vx^{k+1} - \vx^k} \\
		& -\gamma\sqn{\vx_f^{k+1} - \vx^k}
		+ 2\left[h(\vx^*) - h(\vx_f^{k+1})\right] - \mu\sqn{\vx_f^{k+1} - \vx^*} \\
		& + \frac{2(1-\tau)}{\tau}\left[h(\vx_f^{k}) - h(\vx_f^{k+1})\right] %
		\\&
		-\frac{\theta}{\tau}\sqn{\nabla h(\vx_f^{k+1})}
		+\frac{2\theta}{\tau}\left(\sqn{\nabla \phi_\theta^k(\vx_f^{k+1})} - \frac{1}{4\theta^2}\sqn{\vx_f^{k+1} - \vx_g^k}\right)
        \\&
        + \varepsilon(\zeta^2 + \alpha \sqn{\nabla h(\vx_f^{k+1})})
        + \frac{4\theta}{\tau}(\zeta^2 + \alpha \sqn{\nabla h(\vx_g^{k})})
		\\=&
		\left(\frac{1}{\eta}-\gamma + \varepsilon^{-1}\right)\sqn{\vx^k - \vx^*}
		+(\gamma - \mu)\sqn{\vx_f^{k+1} - \vx^*} \\
		& - \gamma\sqn{\vx_f^{k+1} - \vx^k}
		+ \frac{1}{\eta}\sqn{\eta\gamma(\vx_f^{k+1} - \vx^k) - \eta \tilde{\nabla} h(\vx_f^{k+1})} \\
		& - \frac{\theta}{\tau}\sqn{\nabla h(\vx_f^{k+1})}
		+ \frac{2(1-\tau)}{\tau}\left[h(\vx_f^k) - h(\vx^*)\right] \\
        & - \frac{2}{\tau}\left[h(\vx_f^{k+1}) - h(\vx^*)\right]
		\\&
		+\frac{2\theta}{\tau}\left(\sqn{\nabla \phi_\theta^k(\vx_f^{k+1})} - \frac{1}{4\theta^2}\sqn{\vx_f^{k+1} - \vx_g^k}\right)
        \\&
        + (\varepsilon + \frac{4\theta}{\tau}) \zeta^2 +\varepsilon  \alpha \sqn{\nabla h(\vx_f^{k+1})}
        + \frac{4\theta}{\tau} \alpha \sqn{\nabla h(\vx_g^{k})}
\end{align*}
Rearranging terms yields

\begin{align*}
    \frac{1}{\eta}\sqn{\vx^{k+1} - \vx^*} \leq&
		\left(\frac{1}{\eta}-\gamma + 1/\varepsilon \right)\sqn{\vx^k - \vx^*}
		+(\gamma - \mu)\sqn{\vx_f^{k+1} - \vx^*}
		\\&
		+\gamma(2\eta\gamma - 1)\sqn{\vx_f^{k+1} - \vx^k}
		+\left(2\eta + \varepsilon \alpha- \frac{\theta}{\tau}\right)\sqn{\nabla h(\vx_f^{k+1})}
        \\&
        + \eta (\zeta^2 + \alpha \sqn{\nabla h(\vx_f^{k+1})})
		\\&
		+\frac{2(1-\tau)}{\tau}\left[h(\vx_f^k) - h(\vx^*)\right]
		-\frac{2}{\tau}\left[h(\vx_f^{k+1}) - h(\vx^*)\right]
		\\&
		+\frac{2\theta}{\tau}\left(\sqn{\nabla \phi_\theta^k(\vx_f^{k+1})} - \frac{1}{4\theta^2}\sqn{\vx_f^{k+1} - \vx_g^k}\right)
        \\&
        + (\varepsilon + \frac{4\theta}{\tau}) \zeta^2
        + \frac{4\theta}{\tau} \alpha \sqn{\nabla h(\vx_g^{k})}.
    \end{align*}

    Plugging in 
    \begin{align*}
    \sqn{\nabla h(\vx_g^k)} & \le 2 \sqn{\nabla h(\vx_f^{k+1})} + 2\sqn{\nabla h(\vx_f^{k+1}) - \nabla h(\vx_g^{k})} \\
    & \le 2 \sqn{\nabla h(\vx_f^{k+1})} + 2L_h^2\sqn{\vx_f^{k+1} - \vx_g^{k}},
    \end{align*}
    yields
    \begingroup
    \allowdisplaybreaks
    \begin{align*}
        \frac{1}{\eta}\sqn{\vx^{k+1} - \vx^*}
		\leq&
        \left(\frac{1}{\eta}-\gamma + 1/\varepsilon \right)\sqn{\vx^k - \vx^*}
		+(\gamma - \mu)\sqn{\vx_f^{k+1} - \vx^*}
		\\&
		+\gamma(2\eta\gamma - 1)\sqn{\vx_f^{k+1} - \vx^k} \\
		& + \left(2\eta + (\eta+\varepsilon) \alpha- \frac{\theta}{2\tau}(1-8\alpha)\right)\sqn{\nabla h(\vx_f^{k+1})}
		\\&
		+\frac{2(1-\tau)}{\tau}\left[h(\vx_f^k) - h(\vx^*)\right]
		-\frac{2}{\tau}\left[h(\vx_f^{k+1}) - h(\vx^*)\right]
		\\&
		+\frac{2\theta}{\tau}\left(\sqn{\nabla \phi_\theta^k(\vx_f^{k+1})} - (\frac{1}{4\theta^2}- 4L_h^2\alpha)\sqn{\vx_f^{k+1} - \vx_g^k}\right)
        \\&
        + (\varepsilon + \eta + \frac{4\theta}{\tau}) \zeta^2.
	\end{align*}
    \endgroup

Considering 
\(
\gamma = \mu, \quad \varepsilon = \frac{2}{\mu},\quad \eta = \min(\frac{1}{2\mu},\frac{1}{24}\frac{\theta}{\tau}) %
\)
and assuming 
\(
\alpha < \min\left(\frac{\theta \mu}{16 \tau}, \frac{1}{32\theta^2 L_h^2}\right),
\)
ensures
\[\gamma(2\eta\gamma - 1)<0, \qquad \left(2\eta + (\eta+\varepsilon) \alpha- \frac{\theta}{2\tau}(1-8\alpha)\right)<0, \quad \text{and} \qquad \frac{1}{4\theta^2}- 4L_h^2\alpha > \frac{1}{8\theta^2.}\]

Thus, under 
\[
\sqn{\nabla \phi_\theta^k(\vx_f^{k+1})} \le  \frac{1}{8\theta^2}\sqn{\vx_f^{k+1} - \vx_g^k},
\]
we have that
\begin{align*}
        \frac{1}{\eta}\sqn{\vx^{k+1} - \vx^*}
		\leq&
        \left(\frac{1}{\eta}-\frac{\mu}{2}\right)\sqn{\vx^k - \vx^*}
		\\&
		+\frac{2(1-\tau)}{\tau}\left[h(\vx_f^k) - h(\vx^*)\right]
		-\frac{2}{\tau}\left[h(\vx_f^{k+1}) - h(\vx^*)\right]
        \\&
        + (\frac{2}{\mu}  + \frac{5\theta}{\tau}) \zeta^2.
	\end{align*}
Now, denoting $\rho = \min(\tau, \frac{\eta \mu}{2})$, since $\theta = 1/2\sqrt{2}L_p$ the following contraction holds
	\begin{align*}
		\frac{1}{\eta}\sqn{\vx^{k+1} - \vx^*}
		+&
		\frac{2}{\tau}\Bigg[h(\vx_f^{k+1}) - h(\vx^*) - \Big(\frac{1}{\mu} + \frac{1}{\tau L_p}\Big)\zeta^2\Bigg]\\
		&\leq
		\left(1 - \rho\right)
		\left[\frac{1}{\eta}\sqn{\vx^k - \vx^*}
		+\frac{2}{\tau}\left[h(\vx_f^k) - h(\vx^*) - \Big(\frac{1}{\mu} + \frac{1}{\tau L_p}\Big)\zeta^2\right]\right].
	\end{align*}

Maximizing $\rho(\tau) =\min\left(\tau, \frac{3}{4}, \frac{1}{48}\frac{\theta \mu}{\tau}\right)$, by choosing $\tau = \min\left( \frac{\theta \mu}{16 \alpha}, \sqrt{\frac{\theta \mu}{48}}\right) $ yields,  since ${\tau = \Omega(\mu/L_p)}$,
\[
h(\vx_f^K) - h(\vx^*) \le O\Big((1 - \rho)^K R + \frac{\zeta^2}{\mu}\Big),
\]
where $\rho = O\left(\min\left(\frac{1}{\sqrt{\kappa}}, \frac{\mu}{\alpha L_p}\right)\right)$ and $R = \tau^2L_p\|\vx^0 - \vx^*\|^2 + h(\vx^0) - h(\vx^*)$.
\end{proof}

\section{Additional Experiments}
\subsection{Description of the Attacks}

In our experiments we used the following attacks.
\begin{enumerate}
    \item \textbf{Gaussian Attack.} The Byzantine Nodes send an $\cN(0,I)$ gaussian vector at each iteration.
    \item \textbf{Label Flipping.} Every adversarial worker computes its gradient on permuted labels. Since the labels are in $[10]$ of MNIST, the gradients loss correspond to the one computed on the  covariate $c_i$ associated with the label $(y_i-1)\, \mathrm{mod} \,10$.
    
    \item \textbf{(Optimized) A Little Is Enough \citep{baruch2019little}}. The Byzantine nodes send 
    \[
    \text{ALIE}(\vx) = \nabla \cL_{\cH}(\vx) + \tau_{opt} \cdot \overline{\sigma}_{\cH}
    \]
    where $\overline{\sigma}_{\cH}$ is the coordinate-wise standard deviation of $\nabla \cL_i(\vx), i\in \cH$ and $\tau_{opt} $ is optimized through a linear search.
    \item \textbf{(Optimized) Inner Product Manipulation.\citep{xie2020fall}} The Byzantine nodes send 
    \[
    \text{IPM}(\vx) = (1-\tau_{opt})\nabla \cL_{\cH}(\vx).
    \]
    Where $\tau_{opt} $ is optimized through a linear search.
\end{enumerate}

\subsection{Parameter Tuning}

The choice of the parameters of all algorithm has been performed by an hyper-parameter search of the learning-rate in the Byzantine-free setting, and then applied to the Byzantine case. 

To reduce the number of parameters in \cref{alg:acc_inexact_extragradient}, we define $\tau$, $\eta$, $\gamma$ from $\theta$ only using the theoretical values of the non-Byzantine setting. Namely, we took $\gamma=\mu$, $\tau=\min(1, \sqrt{\frac{\theta \mu}{2}})$, $\eta = \min(\frac{1}{2\mu}, \frac{\theta}{2\tau})$, with $\mu$ the l2-regularization of the empirical loss.

\subsection{Empirical Results}

\label{app:experiments}

\subsubsection{Sweep across Heterogeneity Levels and the Proportion of Adversaries}
\label{app:sec:expe_sweep_heter_proportion}

\paragraph{Experimental Setting} In \cref{app:tab:dir01,app:tab:dir10,app:tab:dir50,app:tab:iid} we proceed with experiments on MNIST with $40$ honest clients, a server relying on the Huber aggregation of \cref{app:alg:Huber_Estimator}. We vary the proportion of Byzantine clients from $0$ to $1/3$ and we also vary the Heterogeneity, going from an extreme heterogeneous environment (Dirichlet with $a=0.1$) to an homogeneous environment with i.i.d. samples. Byzantine clients use the ALIE attack.

In each experiment setting, we report the relative asymptotic error -- $\limsup_{t}\frac{\cL_{\cH}(\vx^t)-\cL_{\cH}(\vx^*)}{\cL_{\cH}(\vx^0)-\cL_{\cH}(\vx^*)}$ --, as well as the number of iterations the algorithm took to reach this error level. When the relative asymptotic error is larger than $1$, this corresponds to a setting in which the Byzantine clients have managed to disrupt the optimization, and thus we do not report the number of iterations. Recall that, as illustrated in \cref{plot:logistic_heterogeneous_loss}, the optimization error reaches a plateau after some iterations, and this asymptotic error degrades when the proportion of Byzantine increases. The error we report is the error at this plateau, and the second number in each cell is the number of communications required to reach it. For $f=0$ columns, we display the number of iterations required to reach the $10^{-4}$ error level.

\paragraph{Discussion} From the experiments in \cref{app:tab:dir01,app:tab:dir10,app:tab:dir50,app:tab:iid}, we notice that 1) More heterogeneous datasets lead to algorithm failures for a smaller proportion of Byzantine clients. 2) The robustness of PIGS and D-GD is comparable, and superior to that of the extragradient methods. 3) For $f=0$, the prox-extragrad method outperforms PIGS when the heterogeneity is very high (the similarity condition number $\Delta/\mu$ is large) and, while PIGS is more efficient in a homogeneous environment ($\Delta/\mu$ small), which is consistent with the fact that prox-extragrad is an accelerated similarity method.

\textbf{Stability of Pigs} It follows from the experiments that, if indeed PIGS is faster than the prox-extrafrad method when the heterogeneity is limited, which translates to PIGS being stable with larger step-sizes, it is not the case when the heterogeneity increases, as increases, and using an accelerated method becomes more useful.

We interpret this phenomenon as a (constant) convergence rate overhead due to the acceleration. While this overhead is negligible when the condition number is very large, the point of the similarity-based algorithms is precisely to drastically reduce the condition number (depending on the similarity rather than directly the smoothness). Thus, the overhead becomes non-negligible, and the prox extragradient is not faster than PIGS when the similarity bound is small.

\begin{table}[h!]
\centering
\caption{Results for heterogeneity = Dirichlet $0.1$.}
\label{app:tab:dir01}
\resizebox{\linewidth}{!}{%
\begin{tabular}{l cc cc cc cc cc}
\toprule
\hfill $f/n$ & \multicolumn{2}{c}{$0$} & \multicolumn{2}{c}{$1/41$} & \multicolumn{2}{c}{$1/9$} & \multicolumn{2}{c}{$1/5$} & \multicolumn{2}{c}{$1/3$} \\
Algorithm & error & iter & asy. error & iter & asy. error & iter & asy. error & iter & asy. error & iter \\
\cmidrule(lr){1-1} \cmidrule(lr){2-3} \cmidrule(lr){4-5} \cmidrule(lr){6-7} \cmidrule(lr){8-9} \cmidrule(lr){10-11}
D-GD                  & $1.0{\times}10^{-4}$ & 3099 & $4.7{\times}10^{-1}$ & 9   & $>1$ &  & $>1$ &  & $>1$ &  \\
PIGS                  & $1.0{\times}10^{-4}$ & 760  & $4.8{\times}10^{-1}$ & 2   & $>1$ &   & $>1$ &   & $>1$ &  \\
Acc.\ ExtraGrad       & $1.0{\times}10^{-4}$ & 463  & $9.5{\times}10^{-1}$ & 3   & $>1$ &   & $>1$ &   & $>1$ &   \\
Acc.\ ExtraGrad + Prox & $1.0{\times}10^{-4}$ & 187 & $>1$ &    & $>1$ &   & $>1$ &   & $>1$ &   \\
\bottomrule
\end{tabular}}
\end{table}

\begin{table}[h!]
\centering
\caption{Results for heterogeneity = Dirichlet $1.0$.}
\label{app:tab:dir10}
\resizebox{\linewidth}{!}{%
\begin{tabular}{l cc cc cc cc cc}
\toprule
\hfill $f/n$ & \multicolumn{2}{c}{$0$} & \multicolumn{2}{c}{$1/41$} & \multicolumn{2}{c}{$1/9$} & \multicolumn{2}{c}{$1/5$} & \multicolumn{2}{c}{$1/3$} \\
Algorithm & error & iter & asy. error & iter & asy. error & iter & asy. error & iter & asy. error & iter \\
\cmidrule(lr){1-1} \cmidrule(lr){2-3} \cmidrule(lr){4-5} \cmidrule(lr){6-7} \cmidrule(lr){8-9} \cmidrule(lr){10-11}
D-GD                  & $1.0{\times}10^{-4}$ & 2455 & $4.6{\times}10^{-2}$ & 107 & $8.7{\times}10^{-1}$ & 2 & $>1$ &  & $>1$ &  \\
PIGS                  & $1.0{\times}10^{-4}$ & 143  & $4.7{\times}10^{-2}$ & 9   & $8.9{\times}10^{-1}$ & 1 & $>1$ &   & $>1$ &   \\
Acc.\ ExtraGrad       & $1.0{\times}10^{-4}$ & 355  & $4.7{\times}10^{-2}$ & 123 & $>1$ &  & $>1$ &   & $>1$ &   \\
Acc.\ ExtraGrad + Prox & $1.0{\times}10^{-4}$ & 191 & $1.3{\times}10^{-1}$ & 9  & $>1$ &  & $>1$ &   & $>1$ &   \\
\bottomrule
\end{tabular}}
\end{table}

\begin{table}[h!]
\centering
\caption{Results for heterogeneity = Dirichlet $5.0$.}
\label{app:tab:dir50}
\resizebox{\linewidth}{!}{%
\begin{tabular}{l cc cc cc cc cc}
\toprule
\hfill $f/n$ & \multicolumn{2}{c}{$0$} & \multicolumn{2}{c}{$1/41$} & \multicolumn{2}{c}{$1/9$} & \multicolumn{2}{c}{$1/5$} & \multicolumn{2}{c}{$1/3$} \\
Algorithm & error & iter & asy. error & iter & asy. error & iter & asy. error & iter & asy. error & iter \\
\cmidrule(lr){1-1} \cmidrule(lr){2-3} \cmidrule(lr){4-5} \cmidrule(lr){6-7} \cmidrule(lr){8-9} \cmidrule(lr){10-11}
D-GD                  & $1.0{\times}10^{-4}$ & 2414 & $1.6{\times}10^{-2}$ & 339 & $2.4{\times}10^{-1}$ & 13 & $8.8{\times}10^{-1}$ & 2 & $>1$ &  \\
PIGS                  & $1.0{\times}10^{-4}$ & 31   & $1.6{\times}10^{-2}$ & 10  & $2.5{\times}10^{-1}$ & 2  & $8.9{\times}10^{-1}$ & 1 & $>1$ &    \\
Acc.\ ExtraGrad       & $1.0{\times}10^{-4}$ & 367  & $1.6{\times}10^{-2}$ & 171 & $2.5{\times}10^{-1}$ & 15 & $>1$ &  & $>1$ &   \\
Acc.\ ExtraGrad + Prox & $1.0{\times}10^{-4}$ & 151 & $1.6{\times}10^{-2}$ & 59 & $2.5{\times}10^{-1}$ & 5 & $8.9{\times}10^{-1}$ & 1 & $>1$ &   \\
\bottomrule
\end{tabular}}
\end{table}

\begin{table}[h!]
\centering
\caption{Results for heterogeneity = iid.}
\label{app:tab:iid}
\resizebox{\linewidth}{!}{%
\begin{tabular}{l | cc cc cc cc cc}
\toprule
\hfill $f/n$ & \multicolumn{2}{c}{$0$} & \multicolumn{2}{c}{$1/41$} & \multicolumn{2}{c}{$1/9$} & \multicolumn{2}{c}{$1/5$} & \multicolumn{2}{c}{$1/3$} \\
Algorithm & error & iter & asy. error & iter & asy. error & iter & asy. error & iter & asy. error & iter \\
\cmidrule(lr){1-1} \cmidrule(lr){2-3} \cmidrule(lr){4-5} \cmidrule(lr){6-7} \cmidrule(lr){8-9} \cmidrule(lr){10-11}
D-GD                  & $1.0{\times}10^{-4}$ & 2428 & $3.5{\times}10^{-3}$ & 1038 & $7.7{\times}10^{-2}$ & 56 & $2.9{\times}10^{-1}$ & 11 & $>1$ &  \\
PIGS                  & $1.0{\times}10^{-4}$ & 32   & $3.6{\times}10^{-3}$ & 17   & $7.9{\times}10^{-2}$ & 3  & $2.9{\times}10^{-1}$ & 1  & $>1$ &  \\
Acc.\ ExtraGrad       & $1.0{\times}10^{-4}$ & 367  & $3.6{\times}10^{-3}$ & 267  & $7.9{\times}10^{-2}$ & 97 & $2.9{\times}10^{-1}$ & 11 & $>1$ &  \\
Acc.\ ExtraGrad + Prox & $1.0{\times}10^{-4}$ & 103 & $3.6{\times}10^{-3}$ & 59   & $7.9{\times}10^{-2}$ & 7  & $2.9{\times}10^{-1}$ & 3  & $>1$ &  \\
\bottomrule
\end{tabular}}
\end{table}

\,

\,

\,

\newpage
\subsubsection{Sweep across Attacks and Defenses}
\label{app:exp:attack_defenses}

In \cref{app:plot:aggregator_comparison_mnist_alie,app:plot:aggregator_comparison_mnist_ipm,app:plot:aggregator_comparison_mnist_lf,app:plot:aggregator_comparison_mnist_gaussian} we compare the impact of the averaging rule used in \cref{alg:reduction} on the performance of D-GD, the extragradient method with and without similarity. To do so, we compare the coordinate-wise trimmed mean \citep{yin2018byzantine}, with our multivariate Huber estimator \cref{app:sec:huber}, Krum \citep{blanchard2017machine}, and Covariance-agnostic Filter (CAF) \citep{allouah2025towards} (cf. \cref{app:sec:dimension} for theoretical guarantees of the optimization methods used with CAF). The learning rate of the proximal extragradient method has been divided by $2$ in contrast to the experiments in \cref{app:sec:expe_sweep_heter_proportion} for stability.

Interestingly, we note that the theoretical tightness of CAF translates into better asymptotic error. However, the trajectories of CAF-based algorithms are highly unstable compared to the other ones.

\begin{figure}[h!]
\centering
\includegraphics[width=\textwidth]{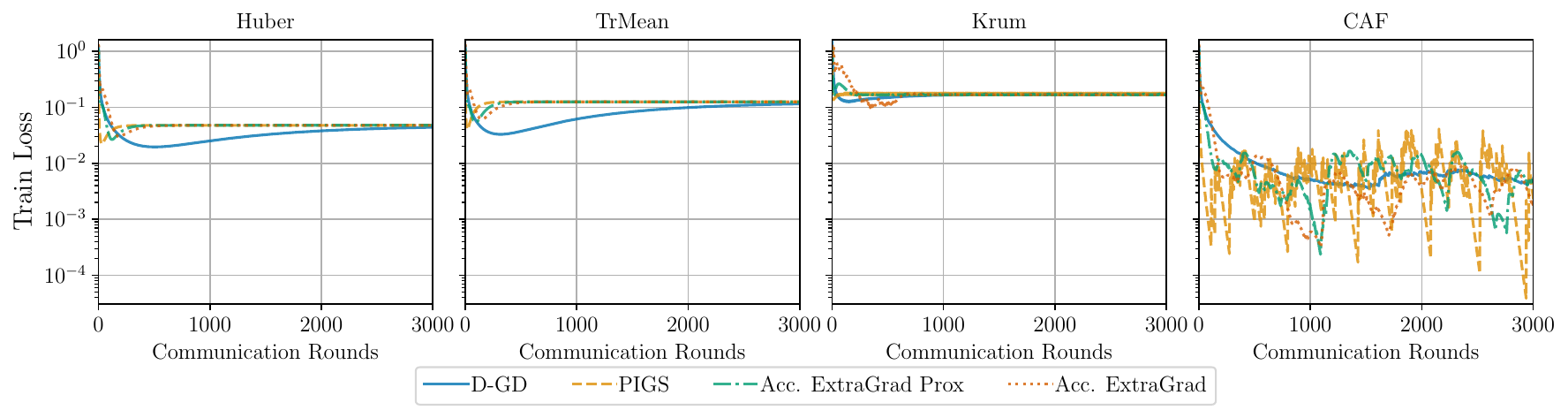}
\caption{Impact of the aggregator under ALIE attack on MNIST dataset (heterogeneity $\beta=1$, strong convexity $\mu=0.01$, $40$ honest clients and $1$ adversary)}
\label{app:plot:aggregator_comparison_mnist_alie}
\end{figure}

\begin{figure}[h!]
\centering
\includegraphics[width=\textwidth]{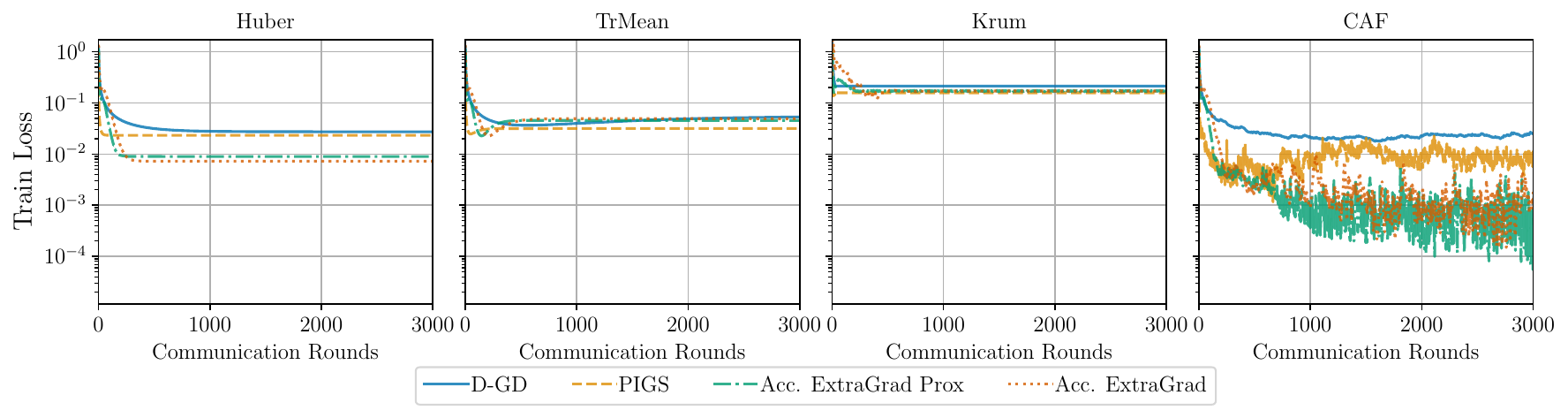}
\caption{Impact of the aggregator under IPM attack on MNIST dataset (heterogeneity $a=1$, strong convexity $\mu=0.01$, $40$ honest clients and $1$ adversary)}
\label{app:plot:aggregator_comparison_mnist_ipm}

\end{figure}

\begin{figure}[h!]
\centering
\includegraphics[width=\textwidth]{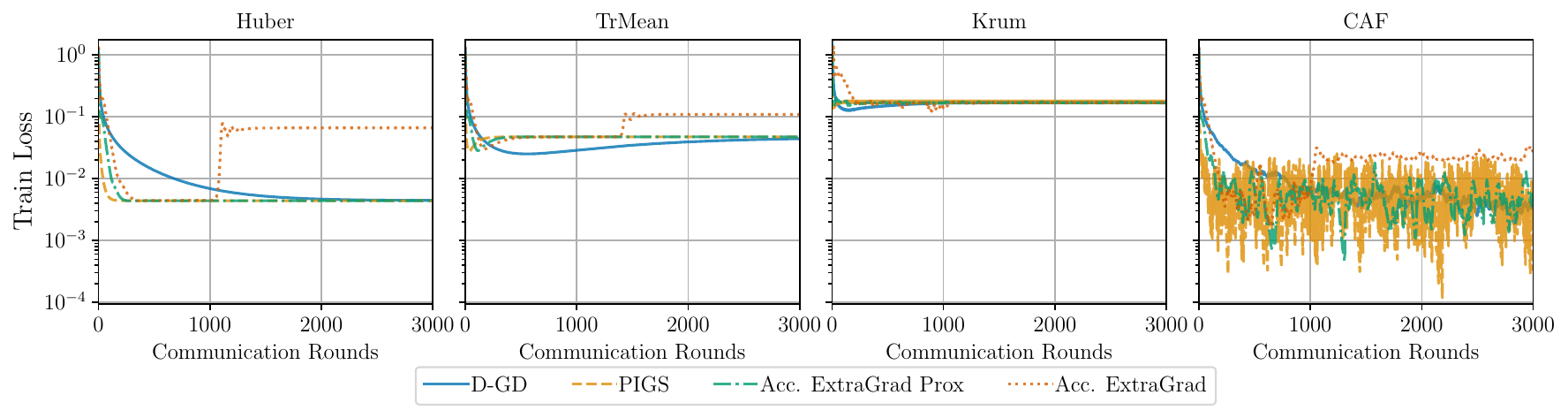}
\caption{Impact of the aggregator under a LabelFlipping attack on MNIST dataset (heterogeneity $a=1$, strong convexity  $\mu=0.01$, $40$ honest clients and $1$ adversary)}
\label{app:plot:aggregator_comparison_mnist_lf}

\end{figure}

\begin{figure}[h!]
\centering
\includegraphics[width=\textwidth]{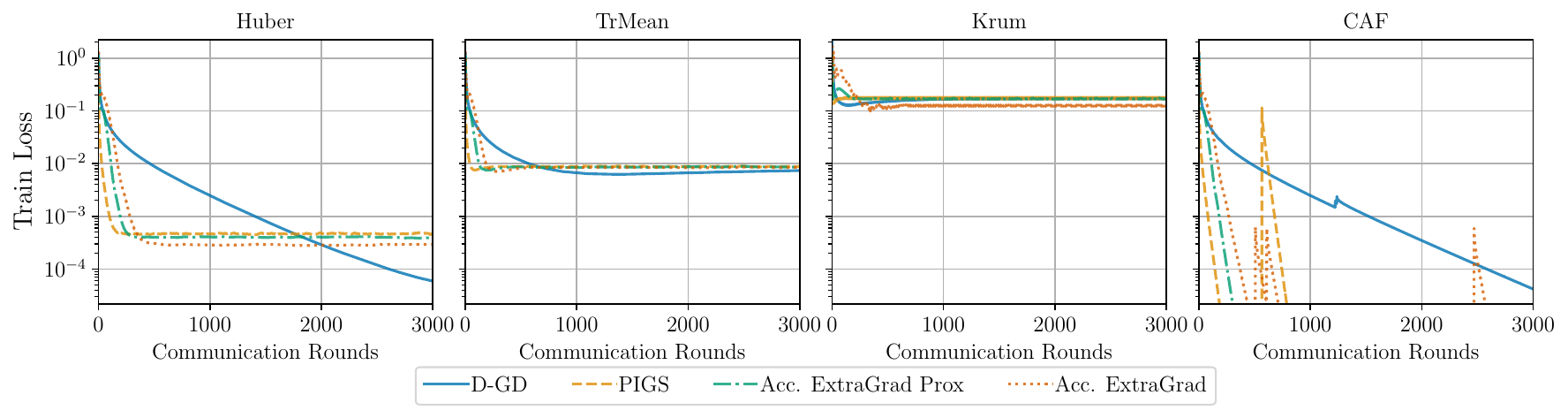}
\caption{Impact of the aggregator under a Gaussian noise attack on MNIST dataset (heterogeneity $a=1$, strong convexity  $\mu=0.01$, $40$ honest clients and $1$ adversary)}
\label{app:plot:aggregator_comparison_mnist_gaussian}

\end{figure}

\section{Additional Discussion}
\label{app:sec:discussion}
\subsection{Limitations}
\label{app:subsec:limitations}

\textbf{Strongly convex regime.} In this work, we chose to focus on designing computational efficient algorithms for first order optimization. Considering the lower bound \cref{thm:allouha_lower_bound}, convergence guaranties either of the loss $\cL(\vx)$ or of the parameter $\|\vx-\vx^*\|^2$ can only be established for strongly convex loss functions. 
Even if it can be shown, for instance, that D-GD with smooth non-convex $\cL_{\cH}$ enjoys $\|\nabla \cL_{\cH}\|^2 = O(\nicefrac{f}{n-2f}G^2)$, it is non-trivial to show accelerated rates without relying on function values. As a consequence, we leave the design and analysis of accelerated Byzantine-robust algorithms beyond the strongly convex regime for future works. 

\textbf{Stochastic Setting.} It is known since \citep{karimireddy2021learning} that permutation-invariant algorithms, i.e. algorithms in which the server average the information from the clients without taking account their identities, can not achieve an optimal error when the clients send i.i.d. stochastic oracles. Namely, when the stochastic gradient oracles send by the clients satisfies $E[\|\nabla \cL_i(\vx)- g_i(\vx)\|^2]\le \sigma^2$, a permutation invariant Byzantine-resilient algorithm can only converge to a $O(\frac{f}{n-2f}\frac{\sigma^2}{\mu})$ error. 

This impossibility result -- which can be understood as a consequence of \cref{thm:allouha_lower_bound} with $B=0$ -- can be alleviated in two ways: either increase the batch size when the algorithm converge, or introduce a momentum term to the stochastic oracle \citep{karimireddy2021learning, farhadkhani2022byzantine}. If plugging our abstraction \cref{alg:acc_inexact_extragradient} in the work of \citet{ajalloeian2020convergence} gives convergence result for Byzantine resilient D-SGD, it does not cover the local momentum strategy.